\documentclass[11pt]{article}
\usepackage[margin=1in]{geometry}
\usepackage{times}
\usepackage[round,authoryear]{natbib}
\usepackage{mathtools,amsfonts,amssymb,bm}
\mathtoolsset{showonlyrefs}
\usepackage{graphicx}
\usepackage{booktabs}
\usepackage{multirow}
\usepackage{makecell}
\usepackage{adjustbox}
\usepackage{nicefrac}
\usepackage{xspace}
\usepackage{pifont}
\usepackage{amsthm}
\usepackage{algorithm}
\usepackage{algpseudocode}

\newcommand{\cmark}{\checkmark}

\newtheorem{theorem}{Theorem}[section]

\newtheorem{proposition}[theorem]{Proposition}

\theoremstyle{definition}

\def\eqref#1{equation~\refeq{#1}}

\def\1{\bm{1}}

\DeclareMathAlphabet{\mathsfit}{\encodingdefault}{\sfdefault}{m}{sl}
\SetMathAlphabet{\mathsfit}{bold}{\encodingdefault}{\sfdefault}{bx}{n}

\usepackage{marvosym}
\usepackage[table]{xcolor}
\usepackage{array}
\usepackage[hyperfootnotes=false]{hyperref} 
\hypersetup{
   colorlinks=true,
   linkcolor={[rgb]{0.10,0.24,0.48}},
   citecolor={[rgb]{0.10,0.24,0.48}},
   urlcolor={[rgb]{0.10,0.24,0.48}},
   pdfborder={0 0 0}
}

\renewenvironment{abstract}{%
   \begin{center}
      \bfseries\abstractname
   \end{center}%
   \begin{center}
      \begin{minipage}{0.90\textwidth}
         \normalsize
         \setlength{\parindent}{0pt}
         \setlength{\parskip}{0.25em}
}{%
      \end{minipage}
   \end{center}
}

\definecolor{headerblue}{RGB}{220,235,248}
\definecolor{oursblue}{RGB}{240,247,252}

\title{Beyond Token-Local Imitation: Reward-Compatible Temporal Credit Assignment for On-Policy Distillation}
\author{
  Shiqi Liu$^{1,2}$, Zeyu He$^{1,2}$, Letian Tao$^{1,2}$, Guojian Zhan$^{1,2}$, Jiaxin Gao$^{1}$, Feihong Zhang$^{1}$\\
  Jingliang Duan$^{1,2}$, Wei Xiong$^{2}$, Kehua Sheng$^{2}$, Bo Zhang$^{2}$, Yang Guan$^{1}$\textsuperscript{\Letter}, Shengbo Eben Li$^{1}$\textsuperscript{\Letter}\\[0.5ex]
  \small $^{1}$School of Vehicle and Mobility \& College of AI, Tsinghua University\\
  \small $^{2}$Didi Voyager Labs, DiDi Autonomous Driving
}
\date{}
\begin{document}

\maketitle

\begingroup
\renewcommand{\thefootnote}{}
\footnotetext{\textsuperscript{\Letter}\ Corresponding authors: S.~E.~Li and Y.~Guan; email: \href{mailto:lishbo@tsinghua.edu.cn}{lishbo@tsinghua.edu.cn}.}
\endgroup

\begin{abstract}
   On-policy distillation (OPD) has emerged as an effective approach for large language model post-training, yet existing objectives face a trade-off between objective fidelity and optimization stability.
   Token-level OPD provides stable but local supervision, whereas sequence-level OPD captures future credit at the cost of horizon-dependent variance.
   We establish a unified temporal-credit view of these formulations, showing that practical token-level OPD can be interpreted as a temporal approximation to the sequence-level reverse-KL gradient.
   Building on this connection, we propose $\gamma$OPD, which uses discounted temporal credit assignment to balance long-horizon supervision and optimization stability, while admitting a horizon-independent variance bound.
   We further develop a reward-compatible bounded mixing (RBM) mechanism for $\gamma\mathrm{OPD}$ that balances verifiable outcome feedback with the discounted OPD advantage to move beyond purely teacher-dependent optimization.
   Experiments on mathematical and code reasoning demonstrate consistent improvements over existing OPD methods across vanilla, size-mismatched, and multi-teacher distillation settings.
   The algorithm implementation is available at
   \href{https://github.com/liu-s-q19/GammaOPD}{\mbox{the GammaOPD repository}}.

\end{abstract}

\begin{figure}[htbp]
   \centering
   \vspace{-2mm}
   \includegraphics[width=0.82\linewidth]{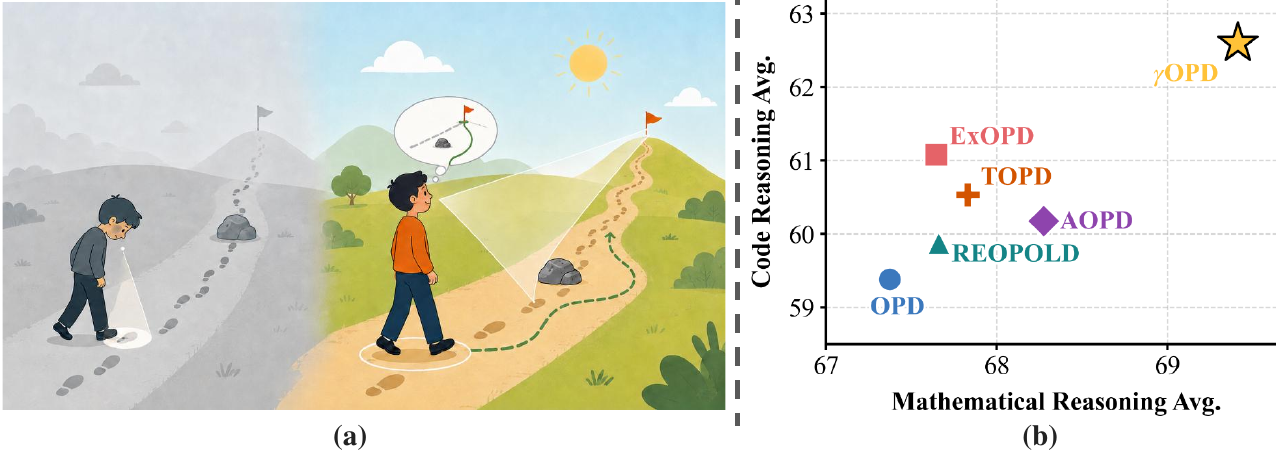}
   \vspace{-2mm}
   \caption{ \textbf{Core Idea.} (a) Vanilla OPD (left) focuses only on the immediate next step and closely follows the teacher's local guidance. We argue that OPD should instead account for the influence of future reasoning steps and should not merely imitate the teacher token by token (right).
      (b) By incorporating future outcomes through temporal credit assignment and balancing teacher supervision with verifiable rewards, $\gamma\mathrm{OPD}$ achieves strong overall performance across mathematical and code reasoning tasks in the multi-teacher distillation setting, as detailed in Section~\ref{sec:experiments}.
   }

   \vspace{-2mm}\label{fig:intro}
\end{figure}

\section{Introduction}

Recent large language models (LLMs), including Kimi K3~\citep{team2026kimi}, GLM-5~\citep{glm5team2026glm5vibecodingagentic}, DeepSeek-V4~\citep{xu2026deepseekv4}, and Nemotron-Cascade 2~\citep{yang2026nemotron}, have demonstrated strong reasoning capabilities across mathematics, coding, and science. To further consolidate and integrate such capabilities, on-policy distillation (OPD)~\citep{lu2025onpolicy} has emerged as a key approach in LLM post-training. It provides dense teacher supervision on trajectories sampled from the current student policy, promoting high-quality reasoning while mitigating exposure bias.

To avoid costly full-vocabulary reverse-KL computation, existing OPD methods \citep{yang2026learning,jin2026entropy,oh2026kl} reformulate the objective in a policy-gradient form and approximate it using a one-sample token-level Monte Carlo estimator~\citep{li2026rethinking}. While computationally efficient, this practical estimator introduces non-negligible bias relative to the original sequence-level objective, effectively altering the optimization problem~\citep{yang2026learning}.
Other methods~\citep{fu2026revisiting} instead adopt sequence-level Monte Carlo estimators to more faithfully preserve the original objective. However, accumulating future credit over the remaining response leads to increasing variance as the sequence grows, resulting in unstable optimization for long reasoning trajectories.

In this work, we present a unified analysis of token-level and sequence-level OPD, deriving their underlying objectives and clarifying the relationship between their policy-gradient formulations. Building on this insight, we propose $\gamma\mathrm{OPD}$, which introduces discounted temporal credit assignment to interpolate between token-local and sequence-level supervision while admitting a horizon-independent variance bound. To complement teacher supervision with outcome-level guidance, we further develop a reward-compatible bounded mixing mechanism that incorporates verifiable outcome feedback without sacrificing the stability benefits of temporal discounting.

Overall, our main contributions are summarized as follows:
\begin{itemize}
   \item We establish a unified temporal-credit view of OPD, showing that practical token-level OPD can be interpreted as a temporally truncated approximation to the sequence-level reverse-KL gradient.

   \item Building on this connection, we propose $\gamma\mathrm{OPD}$, a discounted temporal-credit surrogate that interpolates between token-level and sequence-level credit assignment, while admitting a horizon-independent variance bound.

   \item We further develop Reward-Compatible Bounded Mixing (RBM) for $\gamma\mathrm{OPD}$, which combines discounted teacher-derived credit with verifiable outcome rewards while preventing the teacher signal from dominating task-level supervision.
\end{itemize}

\section{Preliminaries}
\label{sec:preliminaries}

\subsection{Notation}
\label{sec:notation}

We consider autoregressive reasoning tasks over a discrete vocabulary $\mathcal{V}$, where we aim to optimize a student policy $\pi_\theta$ parameterized by $\theta$ under the guidance of a fixed teacher policy $\pi^*$.
Let $\bm{x} \sim \mathcal{D}$ denote a prompt sequence sampled from the reasoning dataset. Given $\bm{x}$, the policy generates a complete response trajectory $\bm{y} = (y_1, y_2, \dots, y_T) \in \mathcal{V}^T$, where $T = |\bm{y}|$ is the number of tokens in the trajectory, and $y_t \in \mathcal{V}$ is the token generated at step $t$.
At decoding step \(t\), the policy conditions on the historical context prefix
\(\bm{h}_t := (\bm{x}, \bm{y}_{<t})\), which consists of the prompt and all previously generated reasoning tokens. We use \(\mathbb{D}_{\mathrm{KL}}(\cdot\|\cdot)\) to denote the KL divergence. For notational simplicity, we define the token-level and sequence-level log-ratios as
\begin{subequations}
   \label{eq:opd-log-ratio}
   \begin{align}
      \Delta_t
       & \triangleq
      \log \pi_\theta(y_t \mid \bm{h}_t)
      -
      \log \pi^*(y_t \mid \bm{h}_t),
      \label{eq:opd-token-log-ratio}
      \\
      \Delta(\bm{y})
       & \triangleq
      \log \pi_\theta(\bm{y} \mid \bm{x})
      -
      \log \pi^*(\bm{y} \mid \bm{x})
      =
      \sum_{t=1}^{T} \Delta_t.
      \label{eq:opd-sequence-log-ratio}
   \end{align}
\end{subequations}

\subsection{On-Policy Distillation (OPD)}

On-policy distillation (OPD) \citep{agarwal2024onpolicy, xu2026deepseekv4}
trains the student by minimizing the reverse KL divergence from the
student policy \(\pi_\theta\) to the teacher policy
\(\pi^*\), evaluated on trajectories induced by the current student
itself:
\begin{equation}
   J_{\mathrm{OPD}}(\theta)
   =
   \mathbb{E}_{\bm{x} \sim \mathcal{D}}
   \left[
      \mathbb{D}_{\mathrm{KL}}\left(\pi_\theta(\bm{y} \mid \bm{x}) \,\|\, \pi^*(\bm{y} \mid \bm{x})\right)
      \right].
   \label{eq:opd-reverse-kl}
\end{equation}
By training on trajectories sampled from the current student policy, on-policy learning reduces exposure bias and is well suited to long chain-of-thought reasoning tasks such as mathematical problem solving \citep{glm5team2026glm5vibecodingagentic}.
Following recent practice \citep{lu2025onpolicy, li2026rethinking}, the empirical implementation of OPD typically converts the reverse-KL minimization in~\eqref{eq:opd-reverse-kl} into an RL-style surrogate gradient:
\begin{equation}
   \begin{aligned}
      \nabla_\theta J_{\mathrm{OPD}}(\theta)
       & \approx -
      \mathbb{E}_{\bm{x} \sim \mathcal{D}}
      \left[
         \sum_{t=1}^{T}
         A_t^{\mathrm{OPD}}
         \nabla_\theta \log \pi_\theta(y_t \mid \bm{h}_t)
         \right],
   \end{aligned}
   \label{eq:opd-policy-gradient}
\end{equation}
where \(A_t^{\mathrm{OPD}}\triangleq -\Delta_t
=
\log \pi^*(y_t \mid \bm{h}_t)
-
\log \pi_\theta(y_t \mid \bm{h}_t)\) is the token-level OPD advantage.

\subsection{RL with Verifiable Rewards (RLVR)}

Reinforcement learning with verifiable rewards (RLVR) \citep{shao2024deepseekmath,guo2025deepseekr1,yu2025dapo} has also become a key approach for improving the reasoning ability of LLMs. For a prompt \(\bm{x} \sim \mathcal{D}\), the policy \(\pi_\theta\) generates an output sequence
\(\bm{y}=(y_1,\ldots,y_T) \sim \pi_\theta(\cdot \mid \bm{x})\). The generated output is evaluated by an external verifier, such as a code compiler or a mathematical rule checker, which provides a sparse sequence-level reward scalar
\(R(\bm{x},\bm{y})\in\{-1,1\}\). The RLVR objective is
\begin{equation}
   J_{\mathrm{RLVR}}(\theta)
   =
   \mathbb{E}_{\bm{x} \sim \mathcal{D},\,
         \bm{y} \sim \pi_\theta(\cdot \mid \bm{x})}
   \left[
      R(\bm{x},\bm{y})
      \right].
   \label{eq:rlvr-objective}
\end{equation}

Compared with OPD, RLVR provides a sparse but verifier-grounded sequence-level correctness signal, whereas OPD supplies dense token-level supervision through teacher guidance, which may nevertheless be limited by the teacher's suboptimal behavior.

\section{Methodology}
\label{sec:methodology}

\subsection{Token-level OPD as a Temporal Approximation}
\label{sec:practical-opd-approximation}

Although the practical OPD objective in~\eqref{eq:opd-policy-gradient} is widely used, it should be understood as a temporal approximation rather than the exact policy gradient of the OPD objective in~\eqref{eq:opd-reverse-kl}. To clarify this distinction, we formally define token-level OPD and sequence-level OPD as follows:
\begin{subequations}
   \label{eq:seq-local-opd-objectives}
   \begin{align}
      J^{\mathrm{seq}}_{\mathrm{OPD}}(\theta)
       & \triangleq
      \mathbb{E}_{\bm{x} \sim \mathcal{D}}
      \left[
         \mathbb{D}_{\mathrm{KL}}
         \left(
         \pi_\theta(\bm{y} \mid \bm{x})
         \,\middle\|\,
         \pi^*(\bm{y} \mid \bm{x})
         \right)
         \right],
      \label{eq:seq-opd-objective}
      \\
      J^{\mathrm{token}}_{\mathrm{OPD}}(\theta)
       & \triangleq
      \mathbb{E}_{\bm{x} \sim \mathcal{D}}
      \left[
         \sum_{t=1}^{T}
         \mathbb{E}_{\bm{h}_t \sim d_{\bar{\theta}}(\cdot \mid \bm{x})}
         \left[
            \mathbb{D}_{\mathrm{KL}}
            \left(
            \pi_\theta(\bm{y}_t \mid \bm{h}_t)
            \,\middle\|\,
            \pi^*(\bm{y}_t \mid \bm{h}_t)
            \right)
            \right]
         \right],
      \label{eq:local-opd-objective}
   \end{align}
\end{subequations}
where \(J^{\mathrm{seq}}_{\mathrm{OPD}}\) denotes the sequence-level reverse-KL objective, i.e., the original OPD objective defined in \eqref{eq:opd-reverse-kl}, which compares the student and teacher distributions over complete response trajectories.
In contrast, \(J^{\mathrm{token}}_{\mathrm{OPD}}\) denotes the token-level objective, which compares their next-token distributions at sampled prefixes. The distribution \(d_{\bar{\theta}}(\bm{h}_t \mid \bm{x})\) is the prefix distribution induced by the current student policy, where \(\bm{h}_t=(\bm{x},\bm{y}_{<t})\). The notation \(\bar{\theta}\) indicates stop-gradient, namely this prefix distribution is treated as fixed when differentiating with respect to \(\theta\).

The two objectives are equivalent at the level of function values when the stop-gradient reference parameter is evaluated at the current policy parameter. Specifically, if \(\bar{\theta}=\theta\), then the sampled prefix distribution in \(J^{\mathrm{token}}_{\mathrm{OPD}}\) matches the autoregressive prefix distribution induced by \(\pi_\theta\), and we have
\begin{equation}
   J^{\mathrm{seq}}_{\mathrm{OPD}}(\theta)
   =
   J^{\mathrm{token}}_{\mathrm{OPD}}(\theta)
   \big|_{\bar{\theta}=\theta}.
   \label{eq:seq-local-value-equivalence}
\end{equation}
However, this value-level equivalence does not imply gradient-level equivalence. The reason is that \(J^{\mathrm{seq}}_{\mathrm{OPD}}\) differentiates through the autoregressive distribution over future prefixes, whereas \(J^{\mathrm{token}}_{\mathrm{OPD}}\) stops this dependence through \(d_{\bar{\theta}}(\bm{h}_t \mid \bm{x})\). This distinction leads to different temporal credit assignments, as shown below.

\begin{proposition}[Token-level OPD gradient]
   \label{prop:token-level-opd-gradient}
   Under the stop-gradient treatment of the prefix distribution
   \(d_{\bar{\theta}}(\bm{h}_t \mid \bm{x})\), the gradient of
   \(J^{\mathrm{token}}_{\mathrm{OPD}}(\theta)\) in \eqref{eq:local-opd-objective} is given by
   \begin{equation}
      \nabla_\theta J^{\mathrm{token}}_{\mathrm{OPD}}(\theta)
      =
      \mathbb{E}_{\bm{x} \sim \mathcal{D},\,
            \bm{y} \sim \pi_\theta(\cdot \mid \bm{x})}
      \left[
         \sum_{t=1}^{T}
         \Delta_t
         \nabla_\theta
         \log \pi_\theta(y_t \mid \bm{h}_t)
         \right].
      \label{eq:token-level-opd-gradient}
   \end{equation}

\end{proposition}

The proof is provided in Appendix~\ref{app:token-level-opd-gradient-proof}.

Consequently, the gradient of
\(J^{\mathrm{token}}_{\mathrm{OPD}}(\theta)\) recovers the practical OPD update in~\eqref{eq:opd-policy-gradient}. Meanwhile, the gradient of the sequence-level OPD objective \(J^{\mathrm{seq}}_{\mathrm{OPD}}(\theta)\) takes the following form:

\begin{proposition}[Sequence-level OPD gradient]
   \label{prop:sequence-level-opd-gradient}
   The gradient of the sequence-level OPD objective
   \(J^{\mathrm{seq}}_{\mathrm{OPD}}(\theta)\) in \eqref{eq:seq-opd-objective} is given by
   \begin{equation}
      \nabla_\theta J^{\mathrm{seq}}_{\mathrm{OPD}}(\theta)
      =
      \mathbb{E}_{\bm{x} \sim \mathcal{D},\,
            \bm{y} \sim \pi_\theta(\cdot \mid \bm{x})}
      \left[
         \sum_{t=1}^{T}
         \left(
         \sum_{t'=t}^{T}
         \Delta_{t'}
         \right)
         \nabla_\theta
         \log \pi_\theta(y_t \mid \bm{h}_t)
         \right].
      \label{eq:sequence-level-opd-gradient}
   \end{equation}
\end{proposition}

The proof is provided in Appendix~\ref{app:sequence-level-opd-gradient-proof}.

Propositions~\ref{prop:token-level-opd-gradient} and~\ref{prop:sequence-level-opd-gradient}
show that practical OPD is a token-level approximation to sequence-level OPD that neglects the influence of subsequent tokens. As illustrated in Figure~\ref{fig:gamma-opd-overview}, under the sequence-level objective, the score term at step \(t\) is weighted by
the future log-ratio \(\sum_{t'=t}^{T}\Delta_{t'}\), since \(y_t\) affects all
future histories \(\bm{h}_{t+1},\ldots,\bm{h}_{T}\). In contrast, practical
OPD keeps only the local term \(\Delta_t\). This gap stems from the
stop-gradient treatment of \(d_{\bar{\theta}}(\bm{h}_t \mid \bm{x})\) in
~\eqref{eq:local-opd-objective}.

\begin{figure}[!t]
   \centering
   \includegraphics[width=0.85\linewidth]{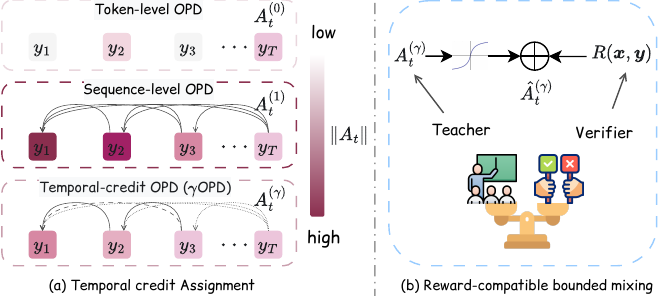}
   \caption{
      \textbf{Overview of \texorpdfstring{$\gamma$}{gamma}OPD.}
      (a) Token-level OPD considers only
      the log-ratio of the current token, whereas sequence-level OPD accumulates the
      log-ratios of all future tokens. $\gamma$OPD introduces a discount factor
      $\gamma$ to interpolate between these two extremes.
      (b) The $\gamma$OPD advantage is
      first normalized and then combined with the verifiable task reward, balancing
      dense teacher guidance with task-level verification signals.}
   \label{fig:gamma-opd-overview}
\end{figure}

\subsection{Temporal-Credit On-Policy Distillation (\texorpdfstring{$\gamma$}{gamma}OPD)}
\label{sec:temporal-credit-opd}

Compared with the token-level gradient in Proposition~\ref{prop:token-level-opd-gradient}, the sequence-level gradient in~\eqref{eq:sequence-level-opd-gradient} accounts for the indirect effect of each token on future autoregressive prefixes, and is therefore unbiased with respect to the sequence-level objective in~\eqref{eq:opd-reverse-kl}. However, for long reasoning trajectories, the accumulated future log-ratio
\(\sum_{t'=t}^{T}\Delta_{t'}\) can have large variance, which may destabilize training.

To balance bias and variance in temporal credit assignment, we propose
\emph{Temporal-Credit On-Policy Distillation ($\gamma$OPD)}. The key idea is to introduce a discount factor
\(\gamma \in [0,1]\) to control how much future OPD credit is assigned to the current token. Specifically, we define the temporal-credit surrogate gradient as
\begin{equation}
   g_{\gamma}(\theta)
   \triangleq -
   \mathbb{E}_{\bm{x} \sim \mathcal{D},\,
         \bm{y} \sim \pi_\theta(\cdot \mid \bm{x})}
   \left[
      \sum_{t=1}^{T}
      A_{t}^{(\gamma)}
      \nabla_\theta
      \log \pi_\theta(y_t \mid \bm{h}_t)
      \right],
   \label{eq:gamma-opd-gradient}
\end{equation}
where the discounted $\gamma\mathrm{OPD}$ advantage is defined as
\begin{equation}
   A_{t}^{(\gamma)}
   \triangleq
   -\sum_{t'=t}^{T}\gamma^{t'-t}\Delta_{t'}.
   \label{eq:gamma-opd-advantage}
\end{equation}
The discount factor \(\gamma\) controls the temporal horizon of OPD credit assignment. When \(\gamma=1\), \(A_{t}^{(1)}\) becomes the sequence-level OPD return-to-go, recovering~\eqref{eq:sequence-level-opd-gradient}. When \(\gamma=0\), it reduces to the token-level OPD advantage \(A_t^{(0)}=-\Delta_t\), recovering~\eqref{eq:token-level-opd-gradient}. Thus, for $0<\gamma<1$, $g_\gamma(\theta)$ is deliberately introduced
as a biased surrogate that interpolates between these two endpoint
gradients, allowing its variance to be controlled through \(\gamma\), as formalized in the following theorem.

\begin{theorem}[Variance Stability of $\gamma\mathrm{OPD}$]
   \label{thm:tc-opd-stability}
   Assume that the token-level OPD advantage has a bounded second moment,
   i.e., \(\mathbb{E}[(A_t^{(0)})^2]\leq\sigma_\Delta^2\) for all \(t\).
   Then, the sequence-level OPD credit admits a horizon-dependent variance
   bound, whereas the $\gamma\mathrm{OPD}$ credit admits a
   horizon-independent variance bound:
   \begin{align}
      \operatorname{Var}\!\left(
      A_t^{(1)}
      \right)
       & \leq
      (T-t+1)^2\sigma_\Delta^2,
      \label{eq:seq-opd-variance-bound}
      \\
      \operatorname{Var}\!\left(
      A_t^{(\gamma)}
      \right)
       & \leq
      \frac{\sigma_\Delta^2}{(1-\gamma)^2},
      \qquad
      \forall \gamma\in[0,1).
      \label{eq:tc-opd-variance-bound}
   \end{align}

\end{theorem}

The proof is provided in Appendix~\ref{app:tc-opd-stability-proof}.

Theorem~\ref{thm:tc-opd-stability} shows that the variance of
sequence-level OPD credit can grow quadratically with the remaining
sequence horizon in the worst case, whereas $\gamma\mathrm{OPD}$
admits a horizon-independent variance bound for any fixed
\(\gamma<1\). Adjusting \(\gamma\) therefore provides a principled
trade-off between training stability and long-horizon credit
propagation.

\subsection{\texorpdfstring{$\gamma\mathrm{OPD}$}{gammaOPD} with Reward-Compatible Bounded Mixing (RBM)}
\label{sec:reward-enhanced-tcopd}

Although $\gamma\mathrm{OPD}$ improves temporal credit assignment, its
optimization signal is still fundamentally derived from the teacher
distribution. Consequently, the student may remain constrained by the
teacher policy and lack an explicit task-level optimization signal for
solving verifiable reasoning problems.

To move beyond this teacher-imitation ceiling, we incorporate the verifiable
task reward $R(\bm{x},\bm{y})\in\{-1,1\}$ into $\gamma\mathrm{OPD}$.
A naive combination, however, can be problematic because the magnitude of
the teacher-derived advantage $A_t^{(\gamma)}$ may vary substantially across
responses and potentially dominate the bounded task reward. We therefore
introduce a \emph{reward-compatible bounded mixing (RBM) mechanism}: the $\gamma\mathrm{OPD}$
credit is first calibrated by its response-wise mean absolute magnitude,
then bounded through a softsign transformation, and finally combined with
the verifiable reward.

These operations can be written compactly as
\begin{equation}
   \widehat{A}_{t}^{(\gamma)}
   \triangleq
   \frac{
      A_{t}^{(\gamma)}
   }{
      \frac{1}{T}\sum_{k=1}^{T}\left|A_{k}^{(\gamma)}\right|
      +
      \left|A_{t}^{(\gamma)}\right|
   } + R(\bm{x},\bm{y}).
   \label{eq:reward-compatible-bounded-tcopd}
\end{equation}
When the denominator vanishes, equivalently when
$A_k^{(\gamma)}=0$ for all $k$, we define the normalized teacher term as
zero.
The first term is equivalent to applying softsign after response-wise
mean-absolute normalization. It is monotonic in $A_t^{(\gamma)}$ and bounded
within $(-1,1)$, thereby retaining the relative token-level OPD credit while
preventing the teacher-derived signal from overwhelming the task reward.
Since $R(\bm{x},\bm{y})\in\{-1,1\}$, it follows that
$\operatorname{sign}(\widehat{A}_{t}^{(\gamma)})=R(\bm{x},\bm{y})$.
Thus, RBM is reward-compatible by construction: the verifiable reward
determines whether the sampled response is reinforced or suppressed, while
the bounded $\gamma\mathrm{OPD}$ credit modulates the token-level update
strength.

The resulting reward-compatible $\gamma\mathrm{OPD}$ surrogate gradient is
\begin{equation}
   \hat{g}_{\gamma}(\theta)
   \triangleq
   -
   \mathbb{E}_{\bm{x}\sim\mathcal{D},\,
         \bm{y}\sim\pi_\theta(\cdot\mid\bm{x})}
   \left[
      \sum_{t=1}^{T}
      \widehat{A}_{t}^{(\gamma)}
      \nabla_\theta
      \log\pi_\theta(y_t\mid\bm{h}_t)
      \right].
   \label{eq:reward-enhanced-tcopd-gradient}
\end{equation}

Therefore, RBM combines sparse verifiable rewards for task-level learning beyond teacher imitation with bounded $\gamma\mathrm{OPD}$ for dense token-level credit assignment.
Unless otherwise specified, $\gamma\mathrm{OPD}$ refers to its use within RBM throughout the remainder of this paper.

\newcommand{\MathResultsTable}{%
\begin{table}[t]
   \centering
   \caption{
      \textbf{Mathematical reasoning results under vanilla and size-mismatch distillation.}
      We report accuracy and pass rate across four mathematical reasoning benchmarks.
   }
   \label{tab:math_distillation}

   \footnotesize
   \setlength{\tabcolsep}{5.5pt}
   \renewcommand{\arraystretch}{1.25}

   \begin{adjustbox}{width=5.5in}
      \begin{tabular}{
         l !{\color{gray!55}\vrule width 0.35pt}
         cc !{\color{gray!55}\vrule width 0.35pt}
         cc !{\color{gray!55}\vrule width 0.35pt}
         cc !{\color{gray!55}\vrule width 0.35pt}
         cc !{\color{gray!55}\vrule width 0.35pt}
         >{\columncolor{gray!10}}c
         >{\columncolor{gray!10}}c
         }

         \toprule

         \multirow{2}{*}{Method}
          & \multicolumn{2}{c}{AIME24}
          & \multicolumn{2}{c}{AIME25}
          & \multicolumn{2}{c}{AMC23}
          & \multicolumn{2}{c}{MATH500}
          & \multicolumn{2}{>{\columncolor{gray!10}}c}{Average}
         \\

         \cmidrule(lr){2-3}
         \cmidrule(lr){4-5}
         \cmidrule(lr){6-7}
         \cmidrule(lr){8-9}
         \cmidrule(lr){10-11}

          & Acc.                                                & Pass
          & Acc.                                                & Pass
          & Acc.                                                & Pass
          & Acc.                                                & Pass
          & Acc.                                                & Pass
         \\
         \specialrule{0.35pt}{0pt}{0pt}

         \rowcolor{headerblue}
         \multicolumn{11}{c}{
            \rule{0pt}{2.6ex}%
            \textbf{Qwen3-4B-Math $\rightarrow$ Qwen3-4B}%
            \rule[-0.8ex]{0pt}{0pt}
         }
         \\[-0.35pt]
         \specialrule{0.35pt}{0pt}{2pt}

         Student
          & 22.60                                               & 60.00
          & 20.83                                               & 33.33
          & 60.16                                               & 92.50
          & 67.40                                               & 72.60
          & 42.75                                               & 64.61
         \\

         Teacher
          & 57.40                                               & 76.67
          & 51.67                                               & 66.67
          & 93.52                                               & 97.50
          & 69.80                                               & 78.40
          & 68.10                                               & 79.81
         \\

         \cmidrule(lr){1-11}

         JustRL
          & 53.54                                               & 83.33
          & 42.50                                               & 56.67
          & 88.75                                               & 97.50
          & 68.00                                               & 70.40
          & 63.20                                               & 76.97
         \\

         STAPO
          & 55.31                                               & 83.33
          & 50.00                                               & 66.67
          & 90.16                                               & 97.50
          & 68.35                                               & 69.80
          & 65.95                                               & 79.33
         \\

         \cmidrule(lr){1-11}

         OPD
          & 56.46                                               & 80.00
          & 50.83                                               & 63.33
          & 92.81                                               & 95.00
          & 68.05                                               & 72.20
          & 67.04                                               & 77.63
         \\

         ExOPD
          & 57.19                                               & 83.33
          & 52.50                                               & 63.33
          & 93.13                                               & \textbf{97.50}
          & 68.90                                               & 74.20
          & 67.93                                               & 79.59
         \\

         REOPOLD
          & 56.98                                               & 83.33
          & 51.67                                               & 63.33
          & \textbf{93.59}                                      & \textbf{97.50}
          & 68.70                                               & 72.60
          & 67.74                                               & 79.19
         \\

         AOPD
          & 56.56                                               & 86.67
          & 52.50                                               & 66.67
          & 93.28                                               & 97.50
          & 69.35                                               & 73.00
          & 67.92                                               & 80.96
         \\

         TOPD
          & 54.79                                               & 80.00
          & 51.67                                               & 63.33
          & 93.44                                               & 97.50
          & 69.95                                               & 72.40
          & 67.46                                               & 78.31
         \\

         \rowcolor{oursblue}
         \textbf{$\gamma$OPD}
          & \textbf{60.94}                                      & \textbf{86.67}
          & \textbf{54.17}                                      & \textbf{66.67}
          & 92.89                                               & \textbf{97.50}
          & \textbf{69.95}                                      & \textbf{78.80}
          & \textbf{69.49}                                      & \textbf{82.41}
         \\

         \specialrule{0.35pt}{0pt}{0pt}
         \rowcolor{headerblue}
         \multicolumn{11}{c}{
            \rule{0pt}{2.6ex}%
            \textbf{Qwen3-4B-Math $\rightarrow$ Qwen3-1.7B}%
            \rule[-0.8ex]{0pt}{0pt}
         }
         \\[-0.35pt]
         \specialrule{0.35pt}{0pt}{2pt}

         Student
          & 11.77                                               & 46.67
          & 8.33                                                & 20.00
          & 39.84                                               & 90.00
          & 59.20                                               & 72.80
          & 29.79                                               & 57.37
         \\

         Teacher
          & 57.40                                               & 76.67
          & 51.67                                               & 66.67
          & 93.52                                               & 97.50
          & 69.80                                               & 78.40
          & 68.10                                               & 79.81
         \\

         \cmidrule(lr){1-11}

         JustRL
          & 37.60                                               & 66.67
          & 36.67                                               & 50.00
          & 78.67                                               & 95.00
          & 65.10                                               & 68.80
          & 54.51                                               & 70.12
         \\

         STAPO
          & 35.52                                               & 63.33
          & 30.83                                               & 46.67
          & 77.73                                               & 95.00
          & 64.95                                               & 68.80
          & 52.26                                               & 68.45
         \\

         \cmidrule(lr){1-11}

         OPD
          & 37.71                                               & 66.67
          & 31.25                                               & 40.00
          & 76.02                                               & 92.50
          & 65.40                                               & 69.80
          & 52.60                                               & 67.24
         \\

         ExOPD
          & 38.54                                               & 66.67
          & 31.67                                               & 43.33
          & 76.72                                               & 92.50
          & 65.05                                               & 68.40
          & 53.00                                               & 67.73
         \\

         REOPOLD
          & 37.50                                               & 66.67
          & 30.83                                               & 43.33
          & 77.73                                               & 92.50
          & 66.70                                               & 69.00
          & 53.19                                               & 67.88
         \\

         AOPD
          & 37.71                                               & 70.00
          & 32.50                                               & 43.33
          & 79.38                                               & 95.00
          & 65.85                                               & 70.80
          & 53.86                                               & 69.78
         \\

         TOPD
          & 39.38                                               & 70.00
          & 30.83                                               & 36.67
          & 79.14                                               & 95.00
          & 65.60                                               & 70.00
          & 53.74                                               & 67.92
         \\

         \rowcolor{oursblue}
         \textbf{$\gamma$OPD}
          & \textbf{42.08}                                      & \textbf{70.00}
          & \textbf{35.00}                                      & \textbf{46.67}
          & \textbf{80.86}                                      & \textbf{95.00}
          & \textbf{68.10}                                      & \textbf{73.20}
          & \textbf{56.51}                                      & \textbf{71.22}
         \\

         \bottomrule
      \end{tabular}%
   \end{adjustbox}
   \vspace{-1em}
\end{table}%
}
\MathResultsTable
\section{Experiments}
\label{sec:experiments}

\subsection{Settings}
\label{sec:experimental-settings}

\textbf{Benchmarks.}
Our experiments evaluate both mathematical and code reasoning abilities. For mathematical reasoning, we train on DeepMath~\citep{he2025deepmath103k}, retaining problems with difficulty level at least 6, and evaluate on AIME24~\citep{li2024numinamath}, AIME25~\citep{opencompass2025aime}, AMC23~\citep{li2024numinamath}, and MATH500~\citep{hendrycks2021math}. For code reasoning, we train on the 25K-sample Eurus-RL-Code dataset~\citep{cui2026process} and evaluate on HumanEval+, MBPP+~\citep{liu2023evalplus}, and the v6 split of LiveCodeBench~\citep{jain2025livecodebench}, covering problems from February 2025 to May 2025. Detailed training and evaluation configurations are provided in Appendix~\ref{app:experiment-details}.

\textbf{Models.}
We conduct experiments under three distillation settings:
(1) \emph{vanilla distillation}, where we distill a Qwen3-4B-Math \citep{yang2026learning} teacher into a Qwen3-4B \citep{yang2025qwen3} student for mathematical reasoning;
(2) \emph{size-mismatch distillation}, where we distill the same Qwen3-4B-Math teacher into a smaller Qwen3-1.7B student for mathematical reasoning; and
(3) \emph{multi-teacher distillation}, where we distill Qwen3-4B-Math and Qwen3-4B-Code~\citep{yang2026learning} teachers into a single Qwen3-4B student for both math and code reasoning.

\textbf{Baselines.}
We compare $\gamma\mathrm{OPD}$ against representative OPD-based baselines, including vanilla OPD~\citep{lu2025onpolicy}, ExOPD~\citep{yang2026learning}, REOPOLD~\citep{ko2026scaling}, AOPD~\citep{jia2026asymmetric}, and TOPD~\citep{zhang2026full}, as well as RL-based methods, including JustRL~\citep{he2026justrl} and STAPO~\citep{liu2026stapo}. For $\gamma\mathrm{OPD}$, we use $\gamma=0.99$ by default unless otherwise specified. All methods are implemented based on \texttt{veRL}~\citep{sheng2025hybridflow}.

\begin{figure}[!t]
   \centering
   \includegraphics[width=0.99\linewidth]{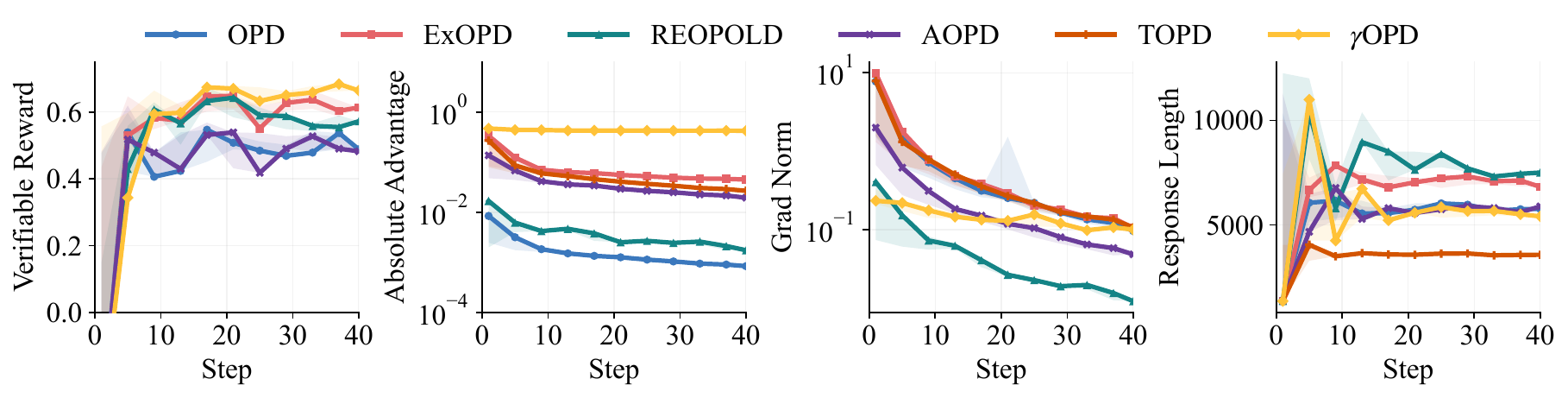}
   \vspace{-1em}
   \caption{
      \textbf{Training dynamics under vanilla distillation.}
      Verifiable reward, absolute advantage, gradient norm, and response length throughout training are reported. Due to the truncation mechanism, TOPD's verifiable reward remains mostly below zero throughout training.
   }
   \label{fig:vanilla-training-dynamics}
   \vspace{-1em}
\end{figure}

\subsection{Main Results}
\label{sec:main-results}

\textbf{Vanilla Distillation.}
As shown in Table~\ref{tab:math_distillation}, $\gamma\mathrm{OPD}$
consistently outperforms existing OPD baselines in the vanilla distillation
setting. Compared with the strongest baseline for each metric, $\gamma\mathrm{OPD}$ achieves relative
improvements of $2.30\%$ in AvgAcc and $1.80\%$ in AvgPass. Notably, it also surpasses
the teacher on both averaged metrics. These results demonstrate the
effectiveness of $\gamma\mathrm{OPD}$ and suggest that the proposed
RBM can better leverage dense teacher supervision while enabling improvement beyond direct imitation.

\textbf{Size-Mismatch Distillation.}
The size-mismatch setting (Qwen3-4B-Math $\rightarrow$ Qwen3-1.7B) is more
challenging because the student has substantially lower capacity than the
teacher. Consequently, as shown in Table~\ref{tab:math_distillation}, all
student methods remain below the teacher. Nevertheless, $\gamma\mathrm{OPD}$
substantially narrows the performance gap and achieves the best overall
performance. Compared with the strongest baseline for each metric,
$\gamma\mathrm{OPD}$ achieves relative improvements of $4.92\%$ in AvgAcc and $2.06\%$ in AvgPass.
These results highlight the benefit of temporal credit assignment
when distilling a stronger teacher into a lower-capacity student.

\textbf{Multi-Teacher Distillation.}
We next evaluate whether $\gamma\mathrm{OPD}$ remains effective in the
multi-teacher distillation setting. As shown in
Table~\ref{tab:multi_teacher_results}, it consistently performs strongly
across mathematical reasoning benchmarks and achieves the best results on
HumanEval+ and LCB among the code reasoning tasks. Overall,
$\gamma\mathrm{OPD}$ improves TotalAvg by $1.87\%$ over the strongest
baseline, AOPD. Notably, it also outperforms the teacher on average in both
mathematical and code reasoning. These results demonstrate that
$\gamma\mathrm{OPD}$ remains effective when jointly training a shared student with domain-specialized teachers.

\begin{table}[!t]
   \centering
   \caption{
      \textbf{Accuracy on mathematical and code reasoning benchmarks under multi-teacher distillation.}
      TotalAvg is the equal-weighted average of the mean mathematical score and the mean code score.
   }
   \label{tab:multi_teacher_results}

   \footnotesize
   \setlength{\tabcolsep}{5.1pt}
   \renewcommand{\arraystretch}{1.25}

   \begin{adjustbox}{max width=5.2in,center}
      \begin{tabular}{
         l !{\color{gray!55}\vrule width 0.35pt}
         cccc !{\color{gray!55}\vrule width 0.35pt}
         ccc !{\color{gray!55}\vrule width 0.35pt}
         >{\columncolor{gray!12}}c
         }

         \toprule

         \multirow{2}{*}{Method}
          & \multicolumn{4}{c}{Mathematical Reasoning}
          & \multicolumn{3}{c}{Code Reasoning}
          &
         \\

         \arrayrulecolor{gray!55}
         \cline{2-5}
         \cline{6-8}
         \arrayrulecolor{black}

          & AIME24
          & AIME25
          & AMC23
          & MATH500
          & HumanEval+
          & MBPP+
          & LCB
          & \multirow{-2}{*}{TotalAvg}
         \\

         \midrule

         Student
          & 22.60
          & 20.83
          & 60.16
          & 67.40
          & 75.61
          & 61.90
          & 17.14
          & 47.15
         \\

         Teacher
          & 57.40
          & 51.67
          & 93.52
          & 69.80
          & 82.32
          & 70.63
          & 25.71
          & 63.82
         \\

         \specialrule{0.35pt}{0pt}{0pt}

         OPD
          & 55.21
          & 52.50
          & 93.13
          & 68.65
          & 84.15
          & 67.99
          & 26.00
          & 63.37
         \\

         ExOPD
          & 57.50
          & 52.50
          & 92.34
          & 68.25
          & 85.37
          & \textbf{71.16}
          & 26.71
          & 64.36
         \\

         REOPOLD
          & 58.13
          & 52.08
          & 91.88
          & 68.55
          & 82.32
          & 69.31
          & 28.00
          & 63.77
         \\

         AOPD
          & 55.52
          & \textbf{55.83}
          & 93.05
          & 68.70
          & 84.76
          & 69.05
          & 26.71
          & 64.22
         \\

         TOPD
          & 57.81
          & 51.67
          & 92.34
          & \textbf{69.50}
          & 84.15
          & 69.58
          & 27.86
          & 64.18
         \\

         \rowcolor{oursblue}
         \textbf{$\gamma$OPD}
          & \textbf{59.27}
          & 55.00
          & \textbf{94.38}
          & 69.00
          & \textbf{89.02}
          & 70.63
          & \textbf{28.14}
          & \textbf{66.00}
         \\

         \bottomrule
      \end{tabular}
   \end{adjustbox}
   \vspace{-2em}
\end{table}

\textbf{Training Dynamics.}
We further visualize the training dynamics of different methods under the vanilla distillation setting in Figure~\ref{fig:vanilla-training-dynamics}. Throughout training, $\gamma\mathrm{OPD}$ consistently achieves the highest verifiable reward, indicating the largest proportion of correctly solved problems. It also maintains a stable absolute OPD advantage and exhibits the smallest fluctuations in gradient norm among all methods, suggesting more consistent temporal credit assignment and optimization. Moreover, the response length of $\gamma\mathrm{OPD}$ converges to a shorter and more stable range than those of most baselines, with the exception of TOPD, which explicitly truncates the distillation signal. Together, these results suggest that $\gamma\mathrm{OPD}$ enables more stable optimization while achieving stronger mathematical reasoning performance.

\subsection{Ablation Studies}
\label{sec:ablation}

We isolate the effects of temporal discounting, reward mixing, and bounded normalization on the AIME benchmarks in Table~\ref{tab:opd-ablation}. Temporal discounting alone improves the average accuracy by $2.04$ points over vanilla OPD. Removing temporal discounting from the full method while retaining reward mixing and normalization reduces the gain from $3.91$ to $2.00$ points, confirming its complementary contribution. Moreover, adding naive reward mixing to temporal discounting brings only a marginal $0.05$-point improvement, whereas the full combination including bounded normalization yields the strongest overall performance. These results suggest that temporal credit assignment and reward-compatible normalization provide complementary benefits, with their combination yielding the strongest performance.
We further analyze the computational and memory overhead of the individual components of $\gamma\mathrm{OPD}$. Compared with standard OPD, $\gamma\mathrm{OPD}$ introduces only about $0.1\%$ additional computation time with negligible memory overhead; detailed profiling results are provided in Appendix~\ref{app:internal-overhead-gamma-opd}.

\begin{table}[t]
   \centering
   \caption{
      \textbf{Ablation study of the proposed components on AIME benchmarks.}
      $\gamma$, Mix, and Norm correspond to temporal discounting, reward mixing, and bounded normalization, respectively.
      $\Delta$ Avg. denotes the absolute improvement in average accuracy over the vanilla OPD baseline.
   }
   \label{tab:opd-ablation}
   \setlength{\tabcolsep}{6pt}
   \renewcommand{\arraystretch}{1.12}
   \begin{adjustbox}{width=3.2in,center}
   \begin{tabular}{ccccccc}
      \toprule
      \multicolumn{3}{c}{Components}
               & \multicolumn{3}{c}{Accuracy (\%)}
               & \multirow{2}{*}{$\Delta$ Avg.}                                            \\
      \cmidrule(lr){1-3}
      \cmidrule(lr){4-6}
      $\gamma$ & Mix                               & Norm
               & AIME24                            & AIME25            & Avg.
               &                                                                           \\
      \midrule

      \rowcolor{gray!10}
      \multicolumn{3}{c}{Vanilla OPD}
               & 56.46                             & 50.83             & 53.65
               & --                                                                        \\

      \cmark   & --                                & --
               & 58.29                             & 53.08             & 55.69
               & +2.04                                                                     \\

      \cmark   & \cmark                            & --
               & \underline{58.81}                 & 52.67             & \underline{55.74}
               & +2.09                                                                     \\
      --       & \cmark                            & \cmark
               & 57.55                             & \underline{53.75} & 55.65
               & +2.00                                                                     \\

      \rowcolor{oursblue}
      \cmark   & \cmark                            & \cmark
               & \textbf{60.94}                    & \textbf{54.17}    & \textbf{57.56}
               & \textbf{+3.91}                                                            \\

      \bottomrule
   \end{tabular}
   \end{adjustbox}
\end{table}

\subsection{Detailed Analysis}
\label{sec:detailed-analysis}

\textbf{Sensitivity to $\gamma$.}
We further investigate the sensitivity of temporal credit assignment to the discount factor $\gamma$ without RBM. Specifically, we distill a Qwen3-4B-Math teacher into a Qwen3-1.7B student using different values of $\gamma$, with the results shown in Figure~\ref{fig:gamma-sensitivity}. Here, $\gamma=0$ reduces to local OPD, whereas $\gamma=1$ corresponds to the undiscounted sequence-level return-to-go. Due to the long response horizon, the sequence-level variant produces substantially larger gradient norms and suffers from rapid entropy collapse, resulting in inferior reasoning performance. In contrast, local OPD and the discounted variants steadily improve during training. Among them, $\gamma=0.99$ achieves the best performance, outperforming smaller values such as $\gamma=0.9$, while exhibiting distinct gradient-norm and entropy dynamics. These results suggest that a large but sub-unity discount factor provides the best balance between long-range temporal credit assignment and optimization stability.

\begin{figure}[!t]
   \centering
   \includegraphics[width=0.9\linewidth]{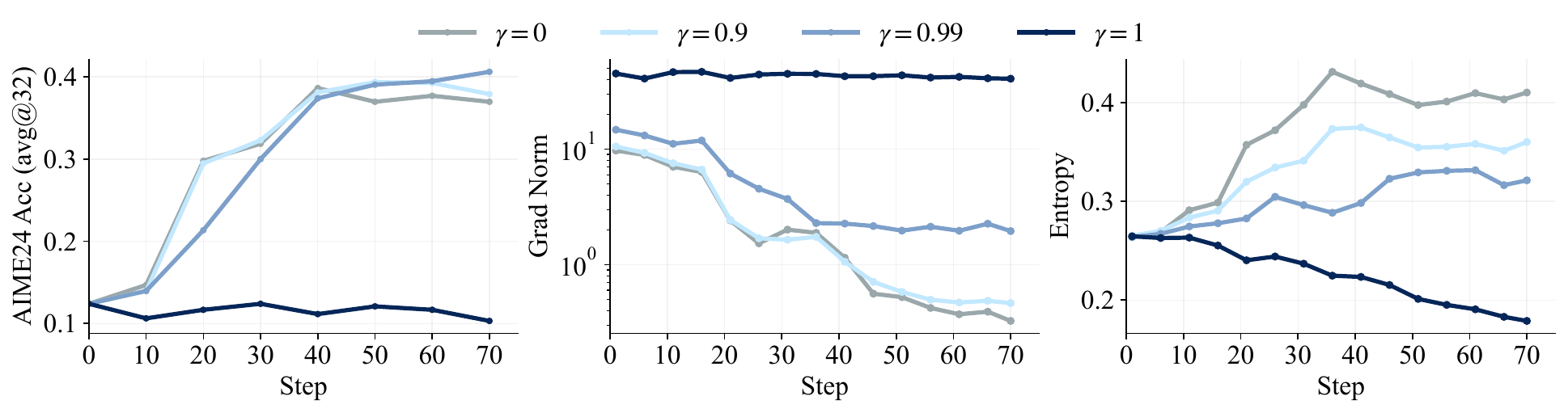}
   \caption{
      \textbf{Training dynamics under different \(\gamma\).}
      OPD distills a Qwen3-4B-Math teacher into a Qwen3-1.7B student, with validation accuracy, gradient norm, and policy entropy reported.
   }
   \label{fig:gamma-sensitivity}
\end{figure}

\begin{figure}[!t]
   \centering
   \includegraphics[width=0.98\linewidth]{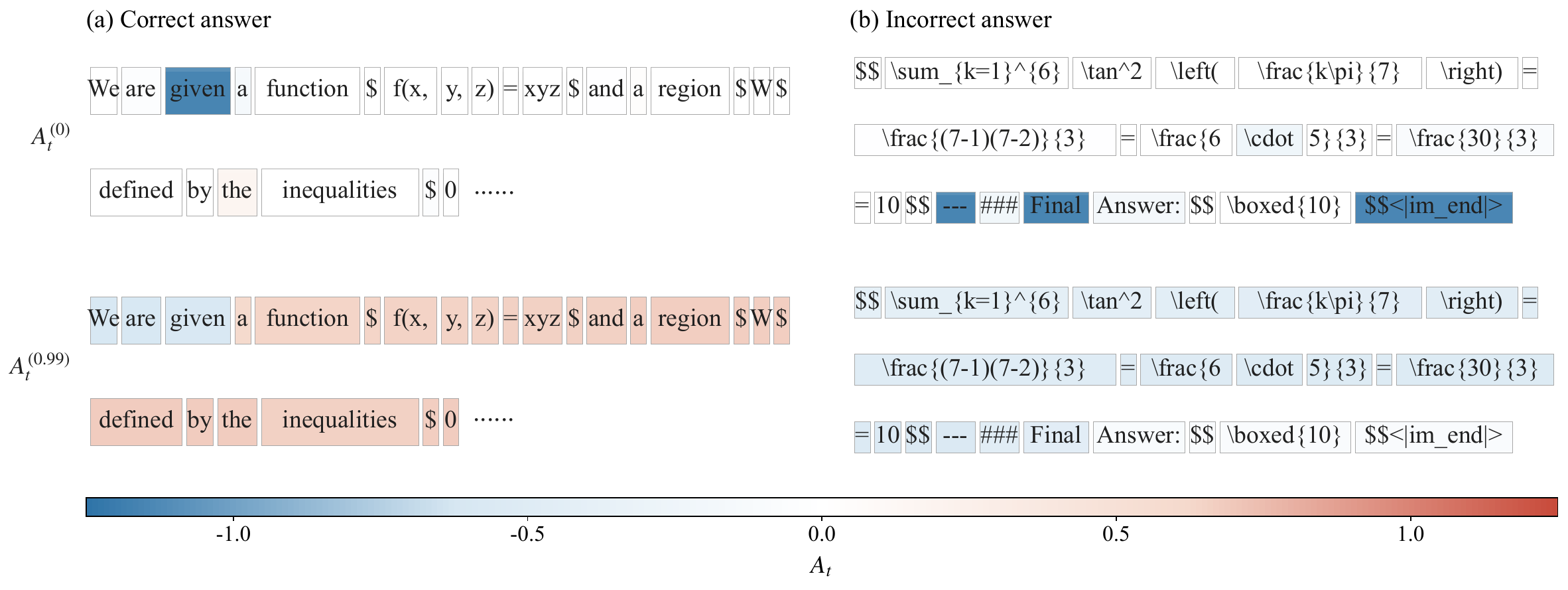}
   \caption{
      \textbf{Token-level advantage visualization.}
      (a) Correct response; (b) incorrect response.
      The top row shows normalized $A_t^{(0)}$, and the bottom row shows normalized $A_t^{(0.99)}$.
      Blue and red denote negative and positive advantages, respectively, with color intensity indicating magnitude.
      Complete response examples are provided in Appendix~\ref{app:complete-token-level-advantages}.
   }
   \label{fig:token-level-advantages}
\end{figure}

\textbf{Visualization of Token Advantages.}
We further visualize the normalized token-level advantages $A_t^{(0)}$ and $A_t^{(0.99)}$ in Figure~\ref{fig:token-level-advantages}. In the correct response, $A_t^{(0)}$ provides sparse local supervision, assigning a noticeable negative signal to only a few tokens while exerting little influence on most of the reasoning process. In contrast, $A_t^{(0.99)}$ propagates information from subsequent steps, producing a smoother signal and assigning stronger positive credit to the key mathematical derivation while mildly penalizing redundant text.
For the incorrect response, $A_t^{(0)}$ strongly penalizes several tokens that are only weakly related to the actual reasoning error, whereas $A_t^{(\gamma)}$ concentrates stronger negative credit on the erroneous formula derivation without excessively penalizing the final token. These examples illustrate that temporal credit assignment can redistribute supervision toward reasoning steps that are more relevant to the final outcome, yielding smoother and more semantically aligned token-level signals. Complete response-level visualizations are provided in Appendix~\ref{app:complete-token-level-advantages}.

\section{Related Work}
\label{sec:related-work}

\textbf{Credit Assignment for OPD.}
Recent studies have explored credit assignment in OPD, motivated by the potentially unreliable reasoning trajectories generated by relatively weak student policies\citep{liu2026stapo,yu2026taming}. Truncation-based methods~\citep{zhou2026less,zhang2026full} terminate rollouts early or mask the learning signals of tokens beyond selected positions, thereby reducing the influence of unreliable continuations. Entropy-aware OPD~\citep{jin2026entropy} augments reverse KL with forward KL on tokens where the teacher has high entropy, while ExOPD~\citep{yang2026learning} introduces a reference model to calibrate the OPD advantage. \citet{fu2026revisiting} analyze the bias--variance trade-off between token-level and sequence-level reverse-KL estimators and propose teacher top-$K$ local-support matching, yet a principled balance between sequence-level objective fidelity and token-level optimization stability remains unresolved.

\textbf{Reward-Guided OPD.}
Recent studies incorporate verifiable rewards into OPD to complement teacher-derived supervision. REOPOLD~\citep{ko2026scaling} combines reward clipping, entropy-based sampling, and exploration-to-refinement scheduling, while SCOPE~\citep{zheng2026scope} routes incorrect trajectories to teacher-perplexity-weighted KL distillation and correct trajectories to student-perplexity-weighted MLE. AOPD~\citep{jia2026asymmetric} preserves positive reinforcement while replacing non-positive token updates with localized teacher imitation, whereas RWOPD~\citep{zou2026rewardweighted} weights teacher KL gradients using verifier rewards. Other works~\citep{ding2026safopd,zhan2026bicriteria} further stabilize optimization through advantage compression and warmup-then-anneal scheduling. Nevertheless, these methods often introduce additional complexity through group-based advantage estimation, log-ratio correction, or verifier-weighted teacher KL objectives, potentially incurring substantial computational overhead.

\section{Conclusion}

In this work, we unified token-level and sequence-level OPD from a temporal credit assignment perspective and proposed $\gamma\mathrm{OPD}$ to balance long-horizon supervision and optimization stability. We further introduced RBM to combine discounted teacher guidance with verifiable outcome rewards. Empirical results on mathematical and code reasoning demonstrate that $\gamma\mathrm{OPD}$ remains effective across different teacher--student configurations and distillation scenarios.

Due to computational constraints, our experiments are currently limited to models with fewer than 10B parameters. Evaluating $\gamma\mathrm{OPD}$ at larger scales and exploring more flexible temporal credit assignment schemes, such as adaptive or entropy-aware discounting, are promising directions for future work.

\bibliography{iclr2027_conference}

@misc{agarwal2024onpolicy,
  author = {Agarwal, Rishabh and Vieillard, Nino and Zhou, Yongchao and Stanczyk, Piotr and Ramos Garea, Sabela and Geist, Matthieu and Bachem, Olivier},
  title = {On-Policy Distillation of Language Models: Learning from Self-Generated Mistakes},
  year = {2024},
  howpublished = {International Conference on Learning Representations}
}

@misc{cui2026process,
  author = {Cui, Ganqu and Yuan, Lifan and Wang, Zefan and Wang, Hanbin and Zhang, Yuchen and Chen, Jiacheng},
  title = {Process Reinforcement through Implicit Rewards},
  year = {2026},
  eprint = {2502.01456},
  archivePrefix = {arXiv}
}

@misc{ding2026safopd,
  author = {Ding, Yifan and Wei, Xincheng and Li, Yoshua Y. and Li, Ziheng and Lu, Yuquan and Zhang, Siyu and Ma, Dongsheng and Weng, Rongxiang and Cai, Xunliang and Chen, Yun},
  title = {SAF-OPD: Stable Advantage Fusion for On-Policy Distillation},
  year = {2026},
  eprint = {2607.29209},
  archivePrefix = {arXiv}
}

@misc{fu2026revisiting,
  author = {Fu, Yuqian and Huang, Haohuan and Jiang, Kaiwen and Liu, Jiacai and Jiang, Zhuo and Zhu, Yuanheng and Zhao, Dongbin},
  title = {Revisiting On-Policy Distillation: Empirical Failure Modes and Simple Fixes},
  year = {2026},
  eprint = {2603.25562},
  archivePrefix = {arXiv}
}

@misc{glm5team2026glm5vibecodingagentic,
  author = {{GLM-5 Team}},
  title = {GLM-5: From Vibe Coding to Agentic Engineering},
  year = {2026},
  eprint = {2602.15763},
  archivePrefix = {arXiv}
}

@misc{guo2025deepseekr1,
  author = {Guo, Daya and Yang, Dejian and Zhang, Haowei and Song, Junxiao and Wang, Peiyi and Zhu, Qihao},
  title = {DeepSeek-R1 Incentivizes Reasoning in LLMs through Reinforcement Learning},
  year = {2025},
  journal = {Nature}
}

@misc{he2025deepmath103k,
  author = {He, Zhiwei and Liang, Tian and Xu, Jiahao and Liu, Qiuzhi and Chen, Xingyu and Wang, Yue and Song, Linfeng},
  title = {DeepMath-103K: A Large-Scale, Challenging, Decontaminated, and Verifiable Mathematical Dataset for Advancing Reasoning},
  year = {2025},
  eprint = {2504.11456},
  archivePrefix = {arXiv}
}

@misc{he2026justrl,
  author = {He, Bingxiang and Qu, Zekai and Liu, Zeyuan and Chen, Yinghao and Zuo, Yuxin and Qian, Cheng and Zhang, Kaiyan},
  title = {JustRL: Scaling a 1.5B LLM with a Simple RL Recipe},
  year = {2026},
  howpublished = {ICLR Blogposts}
}

@misc{hendrycks2021math,
  author = {Hendrycks, Dan and Burns, Collin and Kadavath, Saurav and Arora, Akul and Basart, Steven and Tang, Eric and Song, Dawn and Steinhardt, Jacob},
  title = {Measuring Mathematical Problem Solving with the MATH Dataset},
  year = {2021},
  howpublished = {NeurIPS Datasets and Benchmarks}
}

@misc{jain2025livecodebench,
  author = {Jain, Naman and Han, King and Gu, Alex and Li, Wen-Ding and Yan, Fanjia and Zhang, Tianjun and Wang, Sida and Solar-Lezama, Armando and Sen, Koushik and Stoica, Ion},
  title = {LiveCodeBench: Holistic and Contamination Free Evaluation of Large Language Models for Code},
  year = {2025},
  howpublished = {International Conference on Learning Representations}
}

@misc{jia2026asymmetric,
  author = {Jia, Nan and Yang, Haojin and Ma, Xing and Lian, Jiesong and Zhang, Shuailiang and Zhang, Weipeng and Zeng, Ke and Cai, Xunliang and Sun, Zequn},
  title = {Asymmetric On-Policy Distillation: Bridging Exploitation and Imitation at the Token Level},
  year = {2026},
  eprint = {2605.06387},
  archivePrefix = {arXiv}
}

@misc{jin2026entropy,
  author = {Jin, Woogyeol and Min, Taywon and Yang, Yongjin and Kadhe, Swanand Ravindra and Zhou, Yi and Wei, Dennis and Baracaldo, Nathalie and Lee, Kimin},
  title = {Entropy-Aware On-Policy Distillation of Language Models},
  year = {2026},
  eprint = {2603.07079},
  archivePrefix = {arXiv}
}

@misc{ko2026scaling,
  author = {Ko, Jongwoo and Abdali, Sara and Kim, Young Jin and Chen, Tianyi and Cameron, Pashmina},
  title = {Scaling Reasoning Efficiently via Relaxed On-Policy Distillation},
  year = {2026},
  eprint = {2603.11137},
  archivePrefix = {arXiv}
}

@misc{kwon2023efficient,
  author = {Kwon, Woosuk and Li, Zhuohan and Zhuang, Siyuan and Sheng, Ying and Zheng, Lianmin and Yu, Cody Hao and Gonzalez, Joseph E. and Zhang, Hao and Stoica, Ion},
  title = {Efficient Memory Management for Large Language Model Serving with PagedAttention},
  year = {2023},
  howpublished = {ACM SIGOPS Symposium on Operating Systems Principles}
}

@misc{li2024numinamath,
  author = {Li, Jia and Beeching, Edward and Tunstall, Lewis and Lipkin, Ben and Soletskyi, Roman and Huang, Shengyi and Rasul, Kashif and Yu, Longhui and Jiang, Albert Q. and Shen, Ziju},
  title = {NuminaMath: The Largest Public Dataset in AI4Maths with 860K Pairs of Competition Math Problems and Solutions},
  year = {2024},
  howpublished = {Hugging Face dataset}
}

@misc{li2026rethinking,
  author = {Li, Yaxuan and Zuo, Yuxin and He, Bingxiang and Zhang, Jinqian and Xiao, Chaojun and Qian, Cheng and Yu, Tianyu and Gao, Huanang and Yang, Wenkai and Liu, Zhiyuan and Ding, Ning},
  title = {Rethinking On-Policy Distillation of Large Language Models: Phenomenology, Mechanism, and Recipe},
  year = {2026},
  eprint = {2604.13016},
  archivePrefix = {arXiv}
}

@misc{liu2023evalplus,
  author = {Liu, Jiawei and Xia, Chunqiu Steven and Wang, Yuyao and Zhang, Lingming},
  title = {Is Your Code Generated by ChatGPT Really Correct? Rigorous Evaluation of Large Language Models for Code Generation},
  year = {2023},
  howpublished = {NeurIPS}
}

@misc{liu2026stapo,
  author = {Liu, Shiqi and He, Zeyu and Zhan, Guojian and Tao, Letian and Zheng, Zhilong and Wu, Jiang and Wang, Yinuo and Guan, Yang and Sheng, Kehua and Zhang, Bo and Duan, Jingliang and Li, Shengbo Eben},
  title = {STAPO: Stabilizing Reinforcement Learning for LLMs by Silencing Rare Spurious Tokens},
  year = {2026},
  eprint = {2602.15620},
  archivePrefix = {arXiv}
}

@misc{lu2025onpolicy,
  author = {Lu, Kevin and Thinking Machines Lab},
  title = {On-Policy Distillation},
  year = {2025},
  howpublished = {Thinking Machines Lab blog}
}

@misc{oh2026kl,
  author = {Oh, Minjae and Song, Sangjun and Choi, Gyubin and Choi, Yunho and Jo, Yohan},
  title = {KL for a KL: On-Policy Distillation with Control Variate Baseline},
  year = {2026},
  eprint = {2605.07865},
  archivePrefix = {arXiv}
}

@misc{opencompass2025aime,
  author = {{OpenCompass}},
  title = {AIME2025 Dataset},
  year = {2025},
  howpublished = {Hugging Face dataset}
}

@misc{ren2026radar,
  author = {Ren, et al.},
  title = {RADAR: A Robust Optimizer for Reasoning Model Training},
  year = {2026}
}

@misc{shao2024deepseekmath,
  author = {Shao, Zhihong and Wang, Peiyi and Zhu, Qihao and Xu, Runxin},
  title = {DeepSeekMath: Pushing the Limits of Mathematical Reasoning in Open Language Models},
  year = {2024},
  eprint = {2402.03300},
  archivePrefix = {arXiv}
}

@misc{sheng2025hybridflow,
  author = {Sheng, Guangming and Zhang, Chi and Ye, Zilingfeng and Wu, Xibin and Zhang, Wang and Zhang, Ru and Peng, Yanghua and Lin, Haibin and Wu, Chuan},
  title = {HybridFlow: A Flexible and Efficient RLHF Framework},
  year = {2025},
  howpublished = {EuroSys}
}

@misc{team2026kimi,
  author = {{Kimi Team}},
  title = {Kimi K3: Open Frontier Intelligence},
  year = {2026},
  eprint = {2607.24653},
  archivePrefix = {arXiv}
}

@misc{xu2026deepseekv4,
  author = {{DeepSeek-AI} and Xu, Anyi and Lin, Bangcai and Xue, Bing and Wang, Bingxuan},
  title = {DeepSeek-V4: Towards Highly Efficient Million-Token Context Intelligence},
  year = {2026},
  eprint = {2606.19348},
  archivePrefix = {arXiv}
}

@misc{yang2025qwen3,
  author = {Yang, An and Li, Anfeng and Yang, Baosong and Zhang, Beichen},
  title = {Qwen3 Technical Report},
  year = {2025},
  eprint = {2505.09388},
  archivePrefix = {arXiv}
}

@misc{yang2026learning,
  author = {Yang, Wenkai and Liu, Weijie and Xie, Ruobing and Yang, Kai and Yang, Saiyong and Lin, Yankai},
  title = {Learning Beyond Teacher: Generalized On-Policy Distillation with Reward Extrapolation},
  year = {2026},
  eprint = {2602.12125},
  archivePrefix = {arXiv}
}

@misc{yang2026nemotron,
  author = {Yang, Zhuolin and Liu, Zihan and Chen, Yang and Dai, Wenliang and Wang, Boxin and Lin, Sheng-Chieh and Lee, Chankyu},
  title = {Nemotron-Cascade 2: Post-Training LLMs with Cascade RL and Multi-Domain On-Policy Distillation},
  year = {2026},
  eprint = {2603.19220},
  archivePrefix = {arXiv}
}

@misc{yu2025dapo,
  author = {Yu, Qiying and Zhang, Zheng and Zhu, Ruofei and Yuan, Yufeng},
  title = {DAPO: An Open-Source LLM Reinforcement Learning System at Scale},
  year = {2025},
  howpublished = {NeurIPS}
}

@misc{yu2026taming,
  author = {Yu, et al.},
  title = {Taming Unreliable Reasoning Trajectories in On-Policy Distillation},
  year = {2026}
}

@misc{zhan2026bicriteria,
  author = {Zhan, et al.},
  title = {Bi-Criteria Optimization for Reward-Guided On-Policy Distillation},
  year = {2026}
}

@misc{zhang2026full,
  author = {Zhang, Yaocheng and Chai, Jiajun and Fu, Yuqian and Tu, Songjun and Wang, Xiaohan and Lin, Wei and Yin, Guojun and Zhang, Qichao and Zhu, Yuanheng and Zhao, Dongbin},
  title = {Are Full Rollouts Necessary for On-Policy Distillation?},
  year = {2026},
  eprint = {2605.31490},
  archivePrefix = {arXiv}
}

@misc{zheng2026scope,
  author = {Zheng, Binbin and Ma, Xing and Liang, Yiheng and Ruan, Jingqing and Fu, Xiaoliang and Lin, Kepeng and Zhu, Benchang and Zeng, Ke and Cai, Xunliang},
  title = {SCOPE: Signal-Calibrated On-Policy Distillation Enhancement with Dual-Path Adaptive Weighting},
  year = {2026},
  eprint = {2604.10688},
  archivePrefix = {arXiv}
}

@misc{zhou2026less,
  author = {Zhou, Ziheng and Li, Jiaqi and Tang, Huacong and Wu, Ying Nian and Terzopoulos, Demetri},
  title = {Less Is More: Early Stopping Rollout for On-Policy Distillation},
  year = {2026},
  eprint = {2605.27028},
  archivePrefix = {arXiv}
}

@misc{zou2026rewardweighted,
  author = {Zou, Qingyun and Li, Yingze and Liu, Tianen and He, Bingsheng and Wong, Weng-Fai},
  title = {Reward-Weighted On-Policy Distillation with an Open Property-Equivalence Verifier for NL-to-SVA Generation},
  year = {2026},
  eprint = {2605.13501},
  archivePrefix = {arXiv}
}
\bibliographystyle{iclr2027_conference}

\clearpage

\section*{Appendix}
\appendix
\section{Proof of Proposition~\ref{prop:token-level-opd-gradient}}
\label{app:token-level-opd-gradient-proof}

Starting from the token-level OPD objective in~\eqref{eq:local-opd-objective}, we have
\begin{equation}
   J^{\mathrm{token}}_{\mathrm{OPD}}(\theta)
   =
   \mathbb{E}_{\bm{x} \sim \mathcal{D}}
   \left[
      \sum_{t=1}^{T}
      \mathbb{E}_{\bm{h}_t \sim d_{\bar{\theta}}(\cdot \mid \bm{x})}
      \left[
         \mathbb{D}_{\mathrm{KL}}
         \left(
         \pi_\theta(\cdot \mid \bm{h}_t)
         \,\middle\|\,
         \pi^*(\cdot \mid \bm{h}_t)
         \right)
         \right]
      \right].
   \label{eq:app-local-opd-objective}
\end{equation}
Since the prefix distribution \(d_{\bar{\theta}}(\bm{h}_t \mid \bm{x})\) is treated with stop-gradient, it is fixed when differentiating with respect to \(\theta\). Therefore, the gradient can be moved inside the expectation over prefixes:
\begin{equation}
   \begin{aligned}
      \nabla_\theta J^{\mathrm{token}}_{\mathrm{OPD}}(\theta)
      =
      \mathbb{E}_{\bm{x} \sim \mathcal{D}}
      \left[
         \sum_{t=1}^{T}
         \mathbb{E}_{\bm{h}_t \sim d_{\bar{\theta}}(\cdot \mid \bm{x})}
         \left[
            \nabla_\theta
            \sum_{a \in \mathcal{V}}
            \pi_\theta(a \mid \bm{h}_t)
            \left(
            \log \pi_\theta(a \mid \bm{h}_t)
            -
            \log \pi^*(a \mid \bm{h}_t)
            \right)
            \right]
         \right].
   \end{aligned}
   \label{eq:app-local-gradient-expanded}
\end{equation}
For a fixed prefix \(\bm{h}_t\), define
\begin{equation}
   \Delta(a;\bm{h}_t)
   \triangleq
   \log \pi_\theta(a \mid \bm{h}_t)
   -
   \log \pi^*(a \mid \bm{h}_t).
   \label{eq:app-token-log-ratio-action}
\end{equation}
This is the full-vocabulary counterpart of the sampled token-level log-ratio \(\Delta_t\) in~\eqref{eq:opd-token-log-ratio}. Then the gradient of the inner next-token KL term is
\begin{equation}
   \begin{aligned}
       &
      \nabla_\theta
      \sum_{a \in \mathcal{V}}
      \pi_\theta(a \mid \bm{h}_t)
      \Delta(a;\bm{h}_t)
      \\
       & =
      \sum_{a \in \mathcal{V}}
      \nabla_\theta \pi_\theta(a \mid \bm{h}_t)
      \Delta(a;\bm{h}_t)
      +
      \sum_{a \in \mathcal{V}}
      \pi_\theta(a \mid \bm{h}_t)
      \nabla_\theta \log \pi_\theta(a \mid \bm{h}_t).
   \end{aligned}
   \label{eq:app-local-kl-gradient-decompose}
\end{equation}
The second term in~\eqref{eq:app-local-kl-gradient-decompose} vanishes because
\begin{equation}
   \sum_{a \in \mathcal{V}}
   \pi_\theta(a \mid \bm{h}_t)
   \nabla_\theta \log \pi_\theta(a \mid \bm{h}_t)
   =
   \sum_{a \in \mathcal{V}}
   \nabla_\theta \pi_\theta(a \mid \bm{h}_t)
   =
   \nabla_\theta
   \sum_{a \in \mathcal{V}}
   \pi_\theta(a \mid \bm{h}_t)
   =
   0.
   \label{eq:app-local-score-zero}
\end{equation}
Using the score-function identity
\begin{equation}
   \nabla_\theta \pi_\theta(a \mid \bm{h}_t)
   =
   \pi_\theta(a \mid \bm{h}_t)
   \nabla_\theta \log \pi_\theta(a \mid \bm{h}_t),
   \label{eq:app-local-score-identity}
\end{equation}
we obtain
\begin{equation}
   \begin{aligned}
      \nabla_\theta
      \mathbb{D}_{\mathrm{KL}}
      \left(
      \pi_\theta(\cdot \mid \bm{h}_t)
      \,\middle\|\,
      \pi^*(\cdot \mid \bm{h}_t)
      \right)
       & =
      \sum_{a \in \mathcal{V}}
      \pi_\theta(a \mid \bm{h}_t)
      \Delta(a;\bm{h}_t)
      \nabla_\theta \log \pi_\theta(a \mid \bm{h}_t)
      \\
       & =
      \mathbb{E}_{y_t \sim \pi_\theta(\cdot \mid \bm{h}_t)}
      \left[
         \Delta_t
         \nabla_\theta \log \pi_\theta(y_t \mid \bm{h}_t)
         \right],
   \end{aligned}
   \label{eq:app-local-kl-gradient-score}
\end{equation}
where \(\Delta_t=\Delta(y_t;\bm{h}_t)\) follows from~\eqref{eq:opd-token-log-ratio}.

Substituting~\eqref{eq:app-local-kl-gradient-score} into~\eqref{eq:app-local-gradient-expanded} gives
\begin{equation}
   \begin{aligned}
      \nabla_\theta J^{\mathrm{token}}_{\mathrm{OPD}}(\theta)
      =
      \mathbb{E}_{\bm{x} \sim \mathcal{D}}
      \left[
         \sum_{t=1}^{T}
         \mathbb{E}_{\bm{h}_t \sim d_{\bar{\theta}}(\cdot \mid \bm{x}),\,
               y_t \sim \pi_\theta(\cdot \mid \bm{h}_t)}
         \left[
            \Delta_t
            \nabla_\theta \log \pi_\theta(y_t \mid \bm{h}_t)
            \right]
         \right].
   \end{aligned}
   \label{eq:app-local-gradient-prefix-token}
\end{equation}
When the stop-gradient prefix distribution is evaluated at the current policy, the prefix \(\bm{h}_t\) together with the next token \(y_t\) can be generated by an on-policy trajectory
\(\bm{y}\sim\pi_\theta(\cdot\mid\bm{x})\), while gradients are not propagated through the prefix-sampling distribution. Hence, \eqref{eq:app-local-gradient-prefix-token} can be written as
\begin{equation}
   \nabla_\theta J^{\mathrm{token}}_{\mathrm{OPD}}(\theta)
   =
   \mathbb{E}_{\bm{x} \sim \mathcal{D},\,
         \bm{y} \sim \pi_\theta(\cdot \mid \bm{x})}
   \left[
      \sum_{t=1}^{T}
      \Delta_t
      \nabla_\theta
      \log \pi_\theta(y_t \mid \bm{h}_t)
      \right],
   \label{eq:app-local-gradient-final}
\end{equation}
which is exactly the token-level OPD gradient in~\eqref{eq:token-level-opd-gradient}. This proves Proposition~\ref{prop:token-level-opd-gradient}.
\section{Proof of Proposition~\ref{prop:sequence-level-opd-gradient}}
\label{app:sequence-level-opd-gradient-proof}

Starting from the sequence-level OPD objective in~\eqref{eq:seq-opd-objective}, we expand the reverse KL over complete response trajectories:
\begin{equation}
   J^{\mathrm{seq}}_{\mathrm{OPD}}(\theta)
   =
   \mathbb{E}_{\bm{x} \sim \mathcal{D}}
   \left[
      \sum_{\bm{y} \in \mathcal{V}^{T}}
      \pi_\theta(\bm{y} \mid \bm{x})
      \Delta(\bm{y})
      \right],
   \label{eq:app-seq-opd-expanded}
\end{equation}
where \(\Delta(\bm{y})\) is the sequence-level log-ratio defined in~\eqref{eq:opd-sequence-log-ratio}. Taking the gradient with respect to \(\theta\) gives
\begin{equation}
   \begin{aligned}
      \nabla_\theta J^{\mathrm{seq}}_{\mathrm{OPD}}(\theta)
      =
      \mathbb{E}_{\bm{x} \sim \mathcal{D}}
      \left[
         \sum_{\bm{y} \in \mathcal{V}^{T}}
         \nabla_\theta \pi_\theta(\bm{y} \mid \bm{x})
         \Delta(\bm{y})
         +
         \sum_{\bm{y} \in \mathcal{V}^{T}}
         \pi_\theta(\bm{y} \mid \bm{x})
         \nabla_\theta
         \log \pi_\theta(\bm{y} \mid \bm{x})
         \right].
   \end{aligned}
   \label{eq:app-seq-opd-gradient-expanded}
\end{equation}
The second term in~\eqref{eq:app-seq-opd-gradient-expanded} vanishes because
\begin{equation}
   \sum_{\bm{y} \in \mathcal{V}^{T}}
   \pi_\theta(\bm{y} \mid \bm{x})
   \nabla_\theta \log \pi_\theta(\bm{y} \mid \bm{x})
   =
   \nabla_\theta
   \sum_{\bm{y} \in \mathcal{V}^{T}}
   \pi_\theta(\bm{y} \mid \bm{x})
   =
   \nabla_\theta 1
   =
   0.
   \label{eq:app-score-zero-seq}
\end{equation}
Using the score-function identity,
\begin{equation}
   \nabla_\theta \pi_\theta(\bm{y} \mid \bm{x})
   =
   \pi_\theta(\bm{y} \mid \bm{x})
   \nabla_\theta \log \pi_\theta(\bm{y} \mid \bm{x}),
   \label{eq:app-score-identity-seq}
\end{equation}
we obtain
\begin{equation}
   \nabla_\theta J^{\mathrm{seq}}_{\mathrm{OPD}}(\theta)
   =
   \mathbb{E}_{\bm{x} \sim \mathcal{D},\,
         \bm{y} \sim \pi_\theta(\cdot \mid \bm{x})}
   \left[
      \Delta(\bm{y})
      \nabla_\theta
      \log \pi_\theta(\bm{y} \mid \bm{x})
      \right].
   \label{eq:app-seq-score-form}
\end{equation}
By the autoregressive factorization and the log-ratio decomposition in~\eqref{eq:opd-log-ratio}, we have
\begin{equation}
   \Delta(\bm{y})
   =
   \sum_{t'=1}^{T}
   \Delta_{t'},
   \qquad
   \nabla_\theta
   \log \pi_\theta(\bm{y} \mid \bm{x})
   =
   \sum_{t=1}^{T}
   \nabla_\theta
   \log \pi_\theta(y_t \mid \bm{h}_t).
   \label{eq:app-ar-decomposition-seq}
\end{equation}
Substituting~\eqref{eq:app-ar-decomposition-seq} into~\eqref{eq:app-seq-score-form} yields
\begin{equation}
   \begin{aligned}
      \nabla_\theta J^{\mathrm{seq}}_{\mathrm{OPD}}(\theta)
      =
      \mathbb{E}_{\bm{x} \sim \mathcal{D},\,
            \bm{y} \sim \pi_\theta(\cdot \mid \bm{x})}
      \left[
         \sum_{t=1}^{T}
         \sum_{t'=1}^{T}
         \Delta_{t'}
         \nabla_\theta
         \log \pi_\theta(y_t \mid \bm{h}_t)
         \right].
   \end{aligned}
   \label{eq:app-double-sum-seq}
\end{equation}

It remains to remove the terms with \(t'<t\). For \(t'<t\), \(\Delta_{t'}\) is determined by earlier tokens and is therefore measurable with respect to \(\bm{h}_t\). Hence, conditioning on \(\bm{h}_t\),
\begin{equation}
   \begin{aligned}
       &
      \mathbb{E}_{y_t \sim \pi_\theta(\cdot \mid \bm{h}_t)}
      \left[
         \Delta_{t'}
         \nabla_\theta
         \log \pi_\theta(y_t \mid \bm{h}_t)
         \right]
      \\
       & \quad =
      \Delta_{t'}
      \mathbb{E}_{y_t \sim \pi_\theta(\cdot \mid \bm{h}_t)}
      \left[
         \nabla_\theta
         \log \pi_\theta(y_t \mid \bm{h}_t)
         \right]
      =
      0,
   \end{aligned}
   \label{eq:app-past-term-vanishes}
\end{equation}
where the last equality follows from the token-level score identity
\begin{equation}
   \mathbb{E}_{y_t \sim \pi_\theta(\cdot \mid \bm{h}_t)}
   \left[
      \nabla_\theta
      \log \pi_\theta(y_t \mid \bm{h}_t)
      \right]
   =
   \nabla_\theta
   \sum_{y_t \in \mathcal{V}}
   \pi_\theta(y_t \mid \bm{h}_t)
   =
   0.
   \label{eq:app-token-score-zero}
\end{equation}
Therefore, all terms with \(t'<t\) in~\eqref{eq:app-double-sum-seq} vanish in expectation. Keeping only the terms with \(t'\geq t\), we obtain
\begin{equation}
   \nabla_\theta J^{\mathrm{seq}}_{\mathrm{OPD}}(\theta)
   =
   \mathbb{E}_{\bm{x} \sim \mathcal{D},\,
         \bm{y} \sim \pi_\theta(\cdot \mid \bm{x})}
   \left[
      \sum_{t=1}^{T}
      \left(
      \sum_{t'=t}^{T}
      \Delta_{t'}
      \right)
      \nabla_\theta
      \log \pi_\theta(y_t \mid \bm{h}_t)
      \right],
   \label{eq:app-seq-gradient-final}
\end{equation}
which is exactly the sequence-level gradient in~\eqref{eq:sequence-level-opd-gradient}. This proves Proposition~\ref{prop:sequence-level-opd-gradient}.

\section{Proof of Theorem~\ref{thm:tc-opd-stability}}
\label{app:tc-opd-stability-proof}

\begin{proof}
   Let \(N_t=T-t+1\) and define
   \[
      \widetilde{A}_\tau^{(0)}
      =
      A_\tau^{(0)}
      -
      \mathbb{E}\!\left[A_\tau^{(0)}\right].
   \]
   Since
   \[
      \left\|\widetilde{A}_\tau^{(0)}\right\|_2^2
      =
      \operatorname{Var}\!\left(A_\tau^{(0)}\right)
      \leq
      \mathbb{E}\!\left[\left(A_\tau^{(0)}\right)^2\right]
      \leq
      \sigma_\Delta^2,
   \]
   we have
   \(\left\|\widetilde{A}_\tau^{(0)}\right\|_2\leq\sigma_\Delta\).

   For the sequence-level OPD credit, the triangle inequality in
   \(L_2\) gives
   \begin{align}
      \sqrt{\operatorname{Var}\!\left(A_t^{(1)}\right)}
       & =
      \left\|
      \sum_{\tau=t}^{T}
      \widetilde{A}_\tau^{(0)}
      \right\|_2
      \nonumber \\
       & \leq
      \sum_{\tau=t}^{T}
      \left\|\widetilde{A}_\tau^{(0)}\right\|_2
      \leq
      N_t\sigma_\Delta.
   \end{align}
   Therefore,
   \begin{equation}
      \operatorname{Var}\!\left(A_t^{(1)}\right)
      \leq
      N_t^2\sigma_\Delta^2
      =
      (T-t+1)^2\sigma_\Delta^2.
      \label{eq:appendix-seq-opd-variance-bound}
   \end{equation}

   This quadratic dependence is attainable in the worst case. In
   particular, if \(A_\tau^{(0)}=Z\) for all
   \(\tau\in\{t,\ldots,T\}\), where
   \(\mathbb{E}[Z]=0\) and
   \(\operatorname{Var}(Z)=\sigma_\Delta^2\), then
   \[
      \operatorname{Var}\!\left(A_t^{(1)}\right)
      =
      \operatorname{Var}(N_tZ)
      =
      N_t^2\sigma_\Delta^2.
   \]

   For the \(\gamma\mathrm{OPD}\) credit, similarly,
   \begin{align}
      \sqrt{\operatorname{Var}\!\left(A_t^{(\gamma)}\right)}
       & =
      \left\|
      \sum_{\tau=t}^{T}
      \gamma^{\tau-t}
      \widetilde{A}_\tau^{(0)}
      \right\|_2
      \nonumber \\
       & \leq
      \sigma_\Delta
      \sum_{\tau=t}^{T}
      \gamma^{\tau-t}
      \nonumber \\
       & =
      \sigma_\Delta
      \frac{1-\gamma^{N_t}}{1-\gamma}
      \leq
      \frac{\sigma_\Delta}{1-\gamma},
      \qquad
      \forall\gamma\in[0,1).
   \end{align}
   Squaring both sides yields
   \begin{equation}
      \operatorname{Var}\!\left(A_t^{(\gamma)}\right)
      \leq
      \frac{\sigma_\Delta^2}{(1-\gamma)^2},
      \qquad
      \forall\gamma\in[0,1).
      \label{eq:appendix-gamma-opd-variance-bound}
   \end{equation}

   Thus, sequence-level OPD can exhibit variance that grows
   quadratically with the remaining horizon, whereas
   \(\gamma\mathrm{OPD}\) admits a horizon-independent variance bound
   for every fixed \(\gamma<1\).
\end{proof}

\section{Algorithm}
\label{app:algorithm}

The complete reward-enhanced $\gamma\mathrm{OPD}$ procedure is summarized in Algorithm~\ref{alg:gamma-opd}.
\floatstyle{ruled}
\restylefloat{algorithm}
\begin{algorithm}[H]
   \caption{Reward-Enhanced Temporal-Credit On-Policy Distillation
      (\texorpdfstring{$\gamma$}{gamma}OPD)}
   \label{alg:gamma-opd}
   \begin{algorithmic}[1]
      \Require Dataset $\mathcal{D}$, student policy $\pi_\theta$,
      teacher policy $\pi^*$, discount factor $\gamma$, batch size $B$

      \For{each training iteration}
      \State Sample prompts $\{\bm{x}_i\}_{i=1}^{B}\sim\mathcal{D}$
      and responses $\{\bm{y}_i\}_{i=1}^{B}
         \sim \pi_\theta(\cdot\mid\bm{x}_i)$
      \State Evaluate the verifier to obtain
      $\{R_i\}_{i=1}^{B}$, where
      $R_i \leftarrow R(\bm{x}_i,\bm{y}_i)$

      \For{each mini-batch}
      \For{each response $(\bm{x},\bm{y},R)$ in the mini-batch}
      \For{$t=1,\ldots,T$}
      \State $\Delta_t \leftarrow
         \log\pi_\theta(y_t\mid\bm{h}_t)
         -\log\pi^*(y_t\mid\bm{h}_t)$
      \EndFor
      \State Compute $\{\hat{A}_t^{(\gamma)}\}_{t=1}^{T}$
      according to~\eqref{eq:gamma-opd-advantage} and ~\eqref{eq:reward-compatible-bounded-tcopd}
      \EndFor
      \State Update $\theta$ using~\eqref{eq:reward-enhanced-tcopd-gradient}
      \EndFor
      \EndFor

      \State \textbf{return} final student policy $\pi_\theta$
   \end{algorithmic}
\end{algorithm}

\section{Experiment Details}
\label{app:experiment-details}

All methods, including $\gamma\mathrm{OPD}$, are implemented with \texttt{veRL}~\citep{sheng2025hybridflow}, using vLLM~\citep{kwon2023efficient} for rollouts and FSDP for actor training.
For mathematical reasoning, we train on DeepMath-103K~\citep{he2025deepmath103k}, retaining examples with difficulty level at least 6. Each prompt is formatted in chat style and appended with ``Please output the final answer within \textbackslash boxed\{\}.''
For code reasoning, we train on Eurus-RL-Code~\citep{cui2026process}.
For both domains, verifier outcomes are mapped to binary rewards in $\{+1,-1\}$. Math responses are evaluated with a DAPO-style boxed-answer verifier~\citep{yu2025dapo}, while code responses are verified through execution-based unit tests.
In the multi-teacher setting, each sample is routed to its domain-specific teacher and verifier while sharing the same student policy.

OPD uses reverse-KL token advantages without an additional KL reward penalty. Unless otherwise specified, all models use RADAR~\citep{ren2026radar} optimizer, a constant learning rate on 32 NVIDIA H20 GPUs, with each training run taking approximately 1--3 days; full hyperparameters are provided in Table~\ref{tab:full_hyperparameters}.

\begin{table}[!t]
   \centering
   \caption{Training hyperparameters for the OPD-based methods. The RLVR baselines use a rollout group size of 8.}
   \label{tab:full_hyperparameters}
      \begin{tabular}{lc}
      \toprule
      \textbf{Hyperparameter}          & \textbf{Value}           \\
      \midrule
      Training framework               & \texttt{veRL}            \\
      Rollout backend                  & vLLM                     \\
      Math reward function             & DAPO boxed verifier      \\
      Code reward function             & Execution-based verifier \\
      Train batch size                 & 1024                     \\
      Responses per prompt             & 1                        \\
      PPO mini batch size              & 1024                     \\
      PPO micro batch size/GPU         & 1                        \\
      Max prompt length                & 2048                     \\
      Max response length              & 16384                    \\
      Optimizer                        & RADAR                    \\
      Weight decay                     & 0.01                     \\
      Learning rate                    & $1\times 10^{-5}$        \\
      Training temperature / top-$p$   & 1.0 / 1.0                \\
      Validation temperature / top-$p$ & 1.0 / 1.0                \\
      Rollout tensor parallel size     & 4                        \\
      \bottomrule
   \end{tabular}
   \vspace{-0.45em}
\end{table}

We use temperature 1.0 and top-$p$ 1.0 for all evaluations.
For mathematical reasoning, we evaluate on AIME24~\citep{li2024numinamath}, AIME25~\citep{opencompass2025aime}, AMC23~\citep{li2024numinamath}, and MATH500~\citep{hendrycks2021math}, with a maximum prompt length of 2048 and response length of 16384.
Following the repeated-sampling convention, AIME24, AIME25, and AMC23 are evaluated with 32 samples per problem, while MATH500, MinervaMath, and OlympiadBench use 4 samples per problem; we report both accuracy and pass rate.
All predictions are scored using the same boxed-answer verifier as in training.
For code reasoning, we evaluate HumanEval+ and MBPP+~\citep{liu2023evalplus} with EvalPlus, and the v6 split of LiveCodeBench~\citep{jain2025livecodebench}.
HumanEval+ and MBPP+ use one sample per problem, whereas LiveCodeBench uses four samples with a maximum generation length of 16384 tokens.

\section{Supplementary Experimental Results}
\label{app:supplementary-experimental-results}

\subsection{Computational Time and Memory Analysis}
\label{app:internal-overhead-gamma-opd}

We profile the additional computation introduced by $\gamma\mathrm{OPD}$ using the same 4B student/teacher models and 32-GPU setup as the main experiments. Statistics are averaged over training steps 2--100 after warm-up. The extra operations include discounted temporal-credit computation, bounded advantage shaping, and verifier-reward mixing.
Overall, $\gamma\mathrm{OPD}$ adds only $0.1064\%$ computation time and $0.00046\%$ memory relative to a full OPD update, indicating negligible training overhead.

\begin{table}[!t]
   \centering
   \caption{Computational and memory overhead of the additional
      $\gamma\mathrm{OPD}$ operations.}
   \label{tab:gammaopd_internal_overhead}
   \begin{adjustbox}{max width=\textwidth,center}
      \begin{tabular}{lrr}
         \toprule
         \textbf{Operation}                 &
         \textbf{Time (ms/step)}            &
         \textbf{Memory (MB)}                                \\
         \midrule
         Discounted temporal credit         & 609.85 & 0.063 \\
         Bounded advantage shaping          & 0.42   & 0.125 \\
         Verifier-reward mixing             & 0.04   & 0.188 \\
         \midrule
         \textbf{Total additional overhead} &
         \textbf{610.32}                    &
         \textbf{0.375}                                      \\
         \midrule
         Vanilla OPD update (reference)     &
         573{,}669.00                       &
         81{,}254.40                                         \\
         \midrule
         \textbf{Relative overhead}         &
         \textbf{0.1064\%}                  &
         \textbf{0.00046\%}                                  \\
         \bottomrule
      \end{tabular}
   \end{adjustbox}
   \vspace{-0.45em}
\end{table}

\subsection{Complete Token-Level Advantage Visualization}
\label{app:complete-token-level-advantages}

This appendix provides the complete token-level visualizations behind
Figure~\ref{fig:token-level-advantages}. Each response is shown as a
sequence of consecutive full-width panels rather than as a collection of
small subfigures, so that individual tokens and advantage magnitudes remain
readable. The upper row in every panel shows the local advantage
$A_t^{(0)}$, while the lower row shows the normalized temporally propagated
advantage $A_t^{(0.99)}$.

As shown in Figure~\ref{fig:appendix-correct-token-level-advantages}, for the correct response, $A_t^{(0)}$ gives a sparse local signal, whereas $A_t^{(0.99)}$ propagates information from later reasoning steps and assigns stronger positive credit to the key mathematical derivation. As shown in Figure~\ref{fig:appendix-incorrect-token-level-advantages}, for the incorrect response, $A_t^{(0)}$ penalizes tokens largely unrelated to the actual error, while $A_t^{(0.99)}$ assigns stronger negative credit to the erroneous formula derivation without excessively penalizing the final end-of-sequence token.

{\centering

   \includegraphics[
      page=1,
      width=0.98\textwidth,
      height=0.43\textheight,
      keepaspectratio
   ]{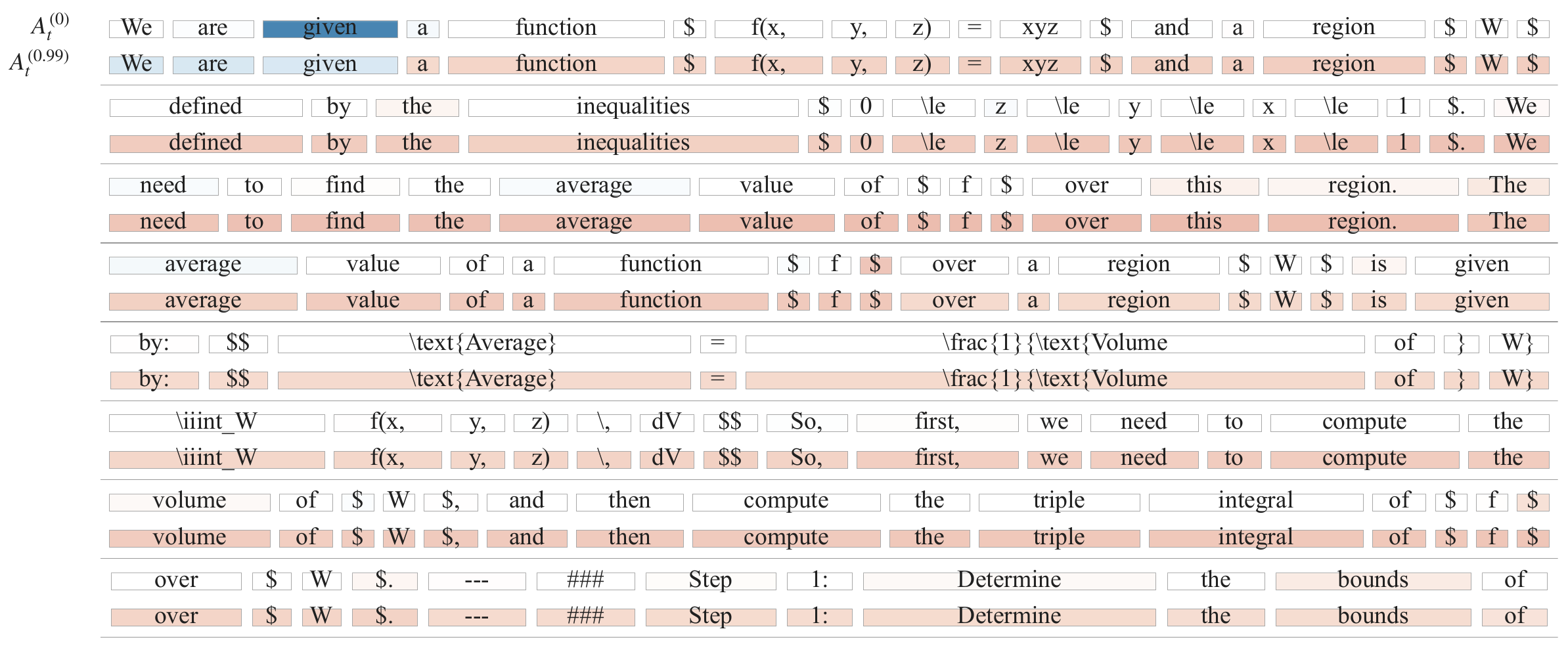}
   \par\vspace{0.2em}

   \includegraphics[
      page=2,
      width=0.98\textwidth,
      height=0.43\textheight,
      keepaspectratio
   ]{figures/ordered-simplex-two-row-complete.pdf}
   \par\vspace{0.2em}

   \includegraphics[
      page=3,
      width=0.98\textwidth,
      height=0.43\textheight,
      keepaspectratio
   ]{figures/ordered-simplex-two-row-complete.pdf}
   \par\vspace{0.2em}

   \includegraphics[
      page=4,
      width=0.98\textwidth,
      height=0.43\textheight,
      keepaspectratio
   ]{figures/ordered-simplex-two-row-complete.pdf}
   \par\vspace{0.2em}

   \includegraphics[
      page=5,
      width=0.98\textwidth,
      height=0.43\textheight,
      keepaspectratio
   ]{figures/ordered-simplex-two-row-complete.pdf}
   \par\vspace{0.2em}

   \includegraphics[
      page=6,
      width=0.98\textwidth,
      height=0.43\textheight,
      keepaspectratio
   ]{figures/ordered-simplex-two-row-complete.pdf}
   \par\vspace{0.2em}

   \includegraphics[
      page=7,
      width=0.98\textwidth,
      height=0.43\textheight,
      keepaspectratio
   ]{figures/ordered-simplex-two-row-complete.pdf}

   \par\vspace{0.8em}

   \refstepcounter{figure}
   \label{fig:appendix-correct-token-level-advantages}

   \parbox{0.96\textwidth}{
      \small
      \textbf{Figure~\thefigure. Complete token-level advantage visualization
         for the correct response.}
      The response is shown in token order across seven consecutive panels,
      with $A_t^{(0)}$ in the upper row and $A_t^{(0.99)}$ in the
      lower row. Blue and red indicate negative and positive advantages,
      respectively.
   }

   \par}

\begin{center}
   \includegraphics[
      page=1,
      width=0.98\textwidth,
      height=0.39\textheight,
      keepaspectratio
   ]{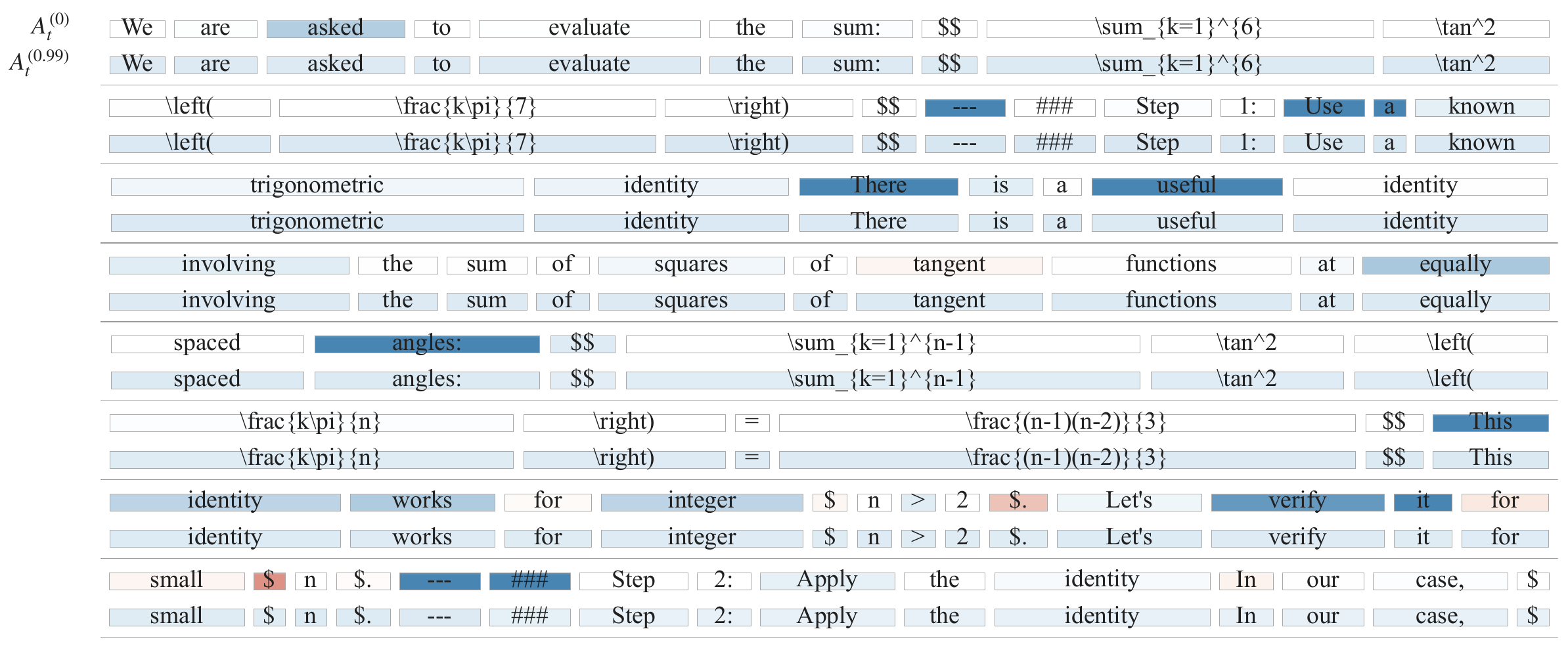}

   \vspace{0.5em}

   \includegraphics[
      page=2,
      width=0.98\textwidth,
      height=0.34\textheight,
      keepaspectratio
   ]{figures/tangent-sum-two-row-complete.pdf}

   \vspace{0.2em}

   \refstepcounter{figure}
   \label{fig:appendix-incorrect-token-level-advantages}
   \parbox{0.96\textwidth}{
      \small
      \textbf{Figure~\thefigure. Complete token-level advantage visualization
         for the incorrect response.}
      The two consecutive panels show how $A_t^{(0)}$ can penalize tokens
      unrelated to the actual error, whereas $A_t^{(0.99)}$ assigns stronger
      negative credit to the erroneous formula derivation.
   }
\end{center}

\end{document}